\documentclass{article} 
\usepackage{iclr2027_conference,times}
\iclrfinalcopy

\usepackage{amsmath,amsfonts,bm}

\def\eqref#1{equation~\ref{#1}}

\def\1{\bm{1}}

\DeclareMathAlphabet{\mathsfit}{\encodingdefault}{\sfdefault}{m}{sl}
\SetMathAlphabet{\mathsfit}{bold}{\encodingdefault}{\sfdefault}{bx}{n}

\usepackage{url}

\usepackage{booktabs}        
\usepackage{wrapfig}
\usepackage{graphicx}
\usepackage{enumitem}
\usepackage{multirow}

\usepackage{amsthm}
\usepackage{amsmath,amssymb}
\usepackage{xcolor}
\usepackage{algorithm}
\usepackage{algpseudocode}
\usepackage{bbding} 
\usepackage{colortbl}
\usepackage{makecell}
  
\definecolor{darkred}{HTML}{b92622}
\definecolor{midnightblue}{HTML}{005c7f}
\definecolor{limegreen}{HTML}{97c65a}
\definecolor{salmon}{HTML}{f1958d}
\definecolor{darkcyan}{HTML}{008B8B}
\definecolor{darkgrey}{rgb}{0.53,0.53,0.53}
\definecolor{mygrey}{rgb}{0.9,0.9,0.9}
\definecolor{cvprblue}{rgb}{0.21,0.49,0.74}
\definecolor{hblue}{rgb}{0.0, 0.0, 1}

\DeclareRobustCommand{\algblue}[1]{\begingroup\setlength{\fboxsep}{0.8pt}\colorbox{cvprblue!18}{\strut #1}\endgroup}
\DeclareRobustCommand{\algred}[1]{\begingroup\setlength{\fboxsep}{0.8pt}\colorbox{orange!8}{\strut #1}\endgroup}
\DeclareRobustCommand{\algbluemath}[1]{\text{\begingroup\setlength{\fboxsep}{0.8pt}\colorbox{cvprblue!18}{$\displaystyle #1$}\endgroup}}
\DeclareRobustCommand{\algredmath}[1]{\text{\begingroup\setlength{\fboxsep}{0.8pt}\colorbox{orange!8}{$\displaystyle #1$}\endgroup}}
\usepackage[most]{tcolorbox}

\newtheorem{proposition}{Proposition}
\usepackage[
    colorlinks=true,
    linkcolor=darkred,
    citecolor=cvprblue,
]{hyperref}

\usepackage[capitalize,noabbrev]{cleveref}
  
\title{Looped Transformers as Optimizers}

\author{Yulong Huang\thanks{Equal contribution. $\textrm{\Envelope}$ Corresponding authors.
Email: \texttt{yhuang496@connect.hkust-gz.edu.cn, leo02@stepfun.com, robert.zhang@stepfun.com, bocheng@hkust-gz.edu.cn}.
}%
$^{~~ 1}$, 
Chen Jiang$^{*~2}$, 
Zhanpeng Zhou$^{3}$,
Hongtao Zhang$^{4}$,
Tianyu Li$^{5}$,\\
\textbf{Tianyu He}$^{2}$,
\textbf{Xiangyu Zhang}$^{\textrm{\scriptsize{\Envelope}}~2}$,
\textbf{Bojun Cheng}$~^{\textrm{\scriptsize{\Envelope}}~1}$
\\
$^1$HKUST (GZ) ~ $^2$StepFun ~ $^3$SJTU ~ $^4$UCAS ~ $^5$THU
}

\begin{document}

\maketitle
\begin{abstract}
Looped Transformers provide a parameter-efficient approach to depth scaling by repeatedly applying shared Transformer blocks. Recent reasoning models have likewise highlighted the value of scaling test-time computation through longer computation trajectories. However, the principles for designing effective loop transitions remain poorly understood. We view the looped hidden state as a fast weight\footnote{Our use of the term \emph{fast weight} is inspired by~\citet{schlag2021transformers}'s work, which interprets the recurrent state as dynamically updated weights that evolve across the sequence. The title of this work pays homage to ~\cite{li2018optimization,schlag2021transformers,dao2024transformers}.} that is updated throughout the depth. We formulate loop transitions as local gradient-based updates, with recurrent blocks predicting implicit targets at each depth. Our framework derives loop transitions in closed form from a projection, a local objective and an optimizer update rule. Mapping representative loop transitions into this framework reveals mismatches between their transitions and projections. We first align the input maps of existing transitions. We then derive OperLoop, which combines explicit weight decay, adaptive step size and a delta objective. The aligned variants reduce training loss and improve average commonsense accuracy. OperLoop improves average generative performance over the compared looped and non-looped baselines under matched training FLOPs. These results support the framework's usefulness for loop design. We extend the analysis to additional loop models and outline a roadmap for future loop transition design.

\end{abstract}

\vspace{-0.4cm}
\section{Introduction}
\vspace{-0.2cm}

The scaling of modern large language models (LLMs) follows an empirical principle~\citep{sutton2019bitter}: jointly increasing model size, training data and computation generally improves performance~\citep{kaplan2020scaling}. In conventional Transformers, increasing depth by stacking independently parameterized blocks also increases the number of unique parameters~\citep{vaswani2017attention}. Greater depth can improve performance on challenging reasoning tasks~\citep{rodkin2026beyond}, but additional parameters can yield diminishing returns under data constraints~\citep{muennighoff2023scaling}. Looped Transformers address this tension by decoupling depth from parameter scaling. They repeatedly reuse shared blocks without proportionally increasing the unique parameters~\citep{dehghani2019universal,lan2020albert}. The repeated computation increases the effective depth, providing additional latent refinement that can enhance the model’s reasoning ability~\citep{saunshi2025reasoning}.

Looped Transformer designs vary along two complementary dimensions: \emph{where to loop} and \emph{how to loop}.
For \emph{where to loop}, recurrence varies in both depth and feedback location. It may span the entire model~\citep{zhu2025ouro}, middle blocks~\citep{geiping2025scaling}, or individual layers~\citep{gao2026loop}, and feedback is drawn from different token or layer locations~\citep{mozer2026recirculation,wang2026full}.
For \emph{how to loop}, methods differ in their recurrent-state update mechanisms, using state decay and source injection for stability~\citep{prairie2026parcae}, or an expanded state with parallel streams~\citep{xie2025mhc} to improve the parameter compute trade-off~\citep{zeitoun2026hyperloop}.


Despite these diverse designs, existing \emph{where} and \emph{how} designs are often coupled to specific architectures. Regardless of loop placement, each iteration must transform the state produced at one step into the
state used at the next. We refer to this transformation as the \emph{recurrent transition}. The transition determines how the state is read, retained and updated, governing representational refinement~\citep{hegazy2026recurrentgpt} and dynamical stability~\citep{prairie2026parcae}.
However, existing transition mechanisms remain closely coupled to their underlying architectures, making it difficult to compare them or derive general principles for improving recurrent state evolution.


In this work, we propose a unified optimization framework for analyzing and designing loop transitions (Figure~\ref{fig:main-overview}). The framework treats the recurrent state as a \emph{fast weight} updated across the depth and connects three components: projection, objective, and optimizer update rule. The projection reads out the fast weight, while shared recurrent blocks implicitly predict the target. The objective relates the projection to a target. A loop transition is therefore a fast weight update governed by the optimizer rule and the objective's gradient. Mapping several loop models into this framework reveals mismatches between their transitions and projections. Aligning these maps improves performance, supporting the framework's design constraints. We then derive OperLoop using this optimization framework. Under matched training FLOPs, OperLoop improves average generative performance over the evaluated looped and non-looped baselines. Together, these results suggest that the framework provides useful guidance for improving existing transitions and designing new ones. Finally, we analyze additional loop models and outline a roadmap for future transition design. Our contributions are:
\begin{itemize}
[leftmargin=2.em]

\item \textbf{Unified Framework.} We develop a unified optimization framework that treats recurrent states as fast weights updated across the depth. It derives the loop transitions in closed form from a projection, a local objective, and an optimizer update rule.

\item \textbf{Input-Map Alignment.} We use the framework to identify input-map mismatches and construct aligned variants of existing looped models. Their improved performance provides empirical support for the framework's design constraints.

\item \textbf{Optimizer-Guided Loop Design.} We derive OperLoop from the framework using a delta objective and adaptive step-size control. Matched compute budget evaluations and component ablations further support the framework's ability to guide effective loop-transition design.

\item \textbf{Extended Analysis and Roadmap.} We apply the framework to additional loop models, identifying shared objective forms and recurring input-map mismatches. This analysis informs a roadmap within the framework to guide future Looped Transformers architecture design.
\end{itemize}

\begin{figure*}[t]
\vspace{-0.6cm}
    \centering
    \includegraphics[trim=0 55 0 45,clip,width=\textwidth]{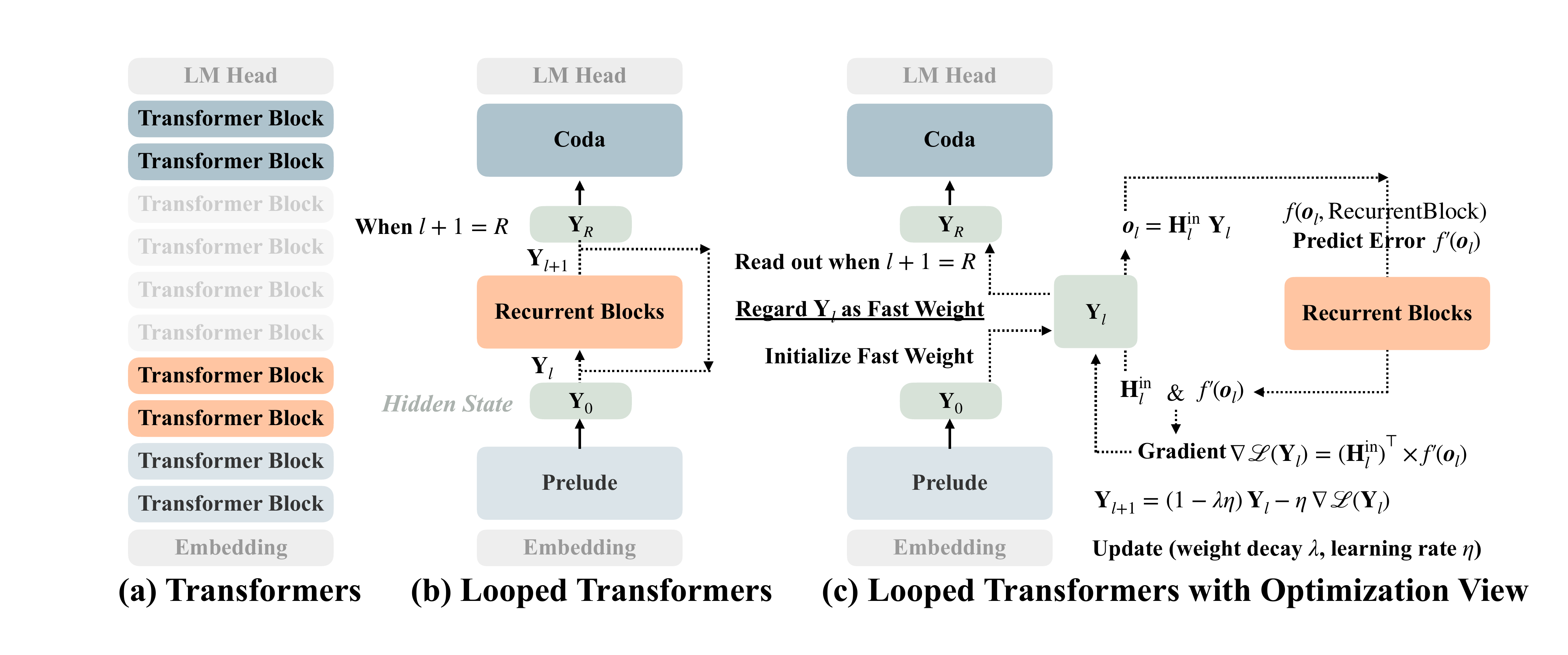}
\caption{Optimization view of Looped Transformers. (a) A standard Transformer uses independently parameterized blocks. (b) A middle-looped Transformer repeatedly applies shared blocks; the Prelude and Coda may be identity maps in a fully looped model. (c) The recurrent state is treated as a fast weight. An input map $\mathbf{H}_l^{\rm in}$ reads out the state, and the shared blocks predict an implicit target. The local objective and optimizer update rule determine the state update. Setting $\mathbf{H}_{\rm in}=\mathbf{I}$, $\lambda=\eta=1$ and $f^{\prime}(\mathbf{o}_l)=-\operatorname{RecurrentBlocks}(\boldsymbol{o}_l)$ recovers the vanilla loop in (b).
}
    \label{fig:main-overview}
\vspace{-0.3cm}
\end{figure*}

\vspace{-0.3cm}
\section{Preliminaries}
\vspace{-0.2cm}
\paragraph{Notation.} Let $\mathbf{X}=(\boldsymbol{x}_1,\ldots,\boldsymbol{x}_N)$ be a sequence of $N$ tokens and $\mathbf{H}\in\mathbb{R}^{N\times d}$ its hidden states, where $d$ is the model width. The symbols $k$ and $l$ denote the depth and loop indices, respectively. We use $L$, $K$, and $R$ for the depth of a conventional Transformer, the number of distinct blocks in the shared recurrent blocks, and the number of recurrent steps, respectively. Unless stated otherwise, bold uppercase and lowercase letters denote matrices and vectors, respectively. We use function composition notation, writing $\mathcal{F}_3(\mathcal{F}_2(\mathcal{F}_1(x)))$ compactly as $\mathcal{F}_3\circ\mathcal{F}_2\circ\mathcal{F}_1(x)$.
\newpage
\paragraph{Standard Transformers.} A standard Transformer increases its effective depth by stacking independently parameterized blocks~\citep{vaswani2017attention}. Given a hidden state $\mathbf{H}^{(k)}$, a pre-norm Transformer block applies an attention sublayer followed by an FFN sublayer:
\begin{equation}
\mathbf{H}^{(k+1)}=\widetilde{\mathbf{H}}^{(k)}+\operatorname{FFN}_{\theta^{\mathrm{FFN}}_k}\!\left(\operatorname{LN}(\widetilde{\mathbf{H}}^{(k)})\right), \qquad \widetilde{\mathbf{H}}^{(k)}=\mathbf{H}^{(k)}+\operatorname{Attn}_{\theta^{\mathrm{Attn}}_k}\!\left(\operatorname{LN}(\mathbf{H}^{(k)})\right).
\end{equation}
For depth index $k=0,\ldots,L-1$, let $\theta_k := (\theta^{\mathrm{Attn}}_k,\theta^{\mathrm{FFN}}_k)$ and $\mathbf{H}^{(k+1)} := \mathcal{F}_{\theta_k}(\mathbf{H}^{(k)})$ for simplicity. Given $\mathbf{H}^{(0)}=\operatorname{Embedding}(\mathbf{X})$, stacking these blocks yields the final hidden state in Eq.~\ref{eq:prelim-dense-composition}:
\begin{equation}
\mathbf{H}^{(L)}=\mathcal{F}_{\theta_{L-1}}\circ\mathcal{F}_{\theta_{L-2}}\circ\cdots\circ\mathcal{F}_{\theta_1}\circ\mathcal{F}_{\theta_0}(\mathbf{H}^{(0)}).
\label{eq:prelim-dense-composition}
\end{equation}

\paragraph{Looped Transformers.}
Looped Transformers increase effective depth by reusing the same blocks without introducing additional parameters~\citep{dehghani2019universal}. Let
$\operatorname{Blocks}_{\theta}
:=\mathcal{F}_{\theta_{i+K-1}}\circ\cdots\circ\mathcal{F}_{\theta_{i+1}}\circ\mathcal{F}_{\theta_{i+0}}$, $\operatorname{Prelude}_\varphi:=\mathcal{F}_{\theta_{i-1}}\circ\cdots\circ\mathcal{F}_{\theta_0}$ and $\operatorname{Coda}_{\phi}:=\mathcal{F}_{\theta_{L-1}}\circ\cdots\circ\mathcal{F}_{\theta_{L-j}}$, where $i$ and $j$ denote the numbers of $\operatorname{Prelude}$ and $\operatorname{Coda}$ blocks, respectively. The middle-looped architecture~\citep{geiping2025scaling} yields the final hidden state in Eq.~\ref{eq:prelim-loop-composition}:
\begin{equation}
    \mathbf{H}^{(L)}
    =\operatorname{Coda}_{\phi}\circ
    \operatorname{Blocks}_{\theta}\circ\cdots\circ\operatorname{Blocks}_{\theta}
    \circ
    \operatorname{Prelude}_\varphi(\mathbf{H}^{(0)}).
    \label{eq:prelim-loop-composition}
\end{equation}
The shared $\operatorname{Blocks}_{\theta}$ is iteratively applied $R$ times to update the latent hidden state. The $\operatorname{Prelude}$ embeds the input representation into the latent space, while the $\operatorname{Coda}$ decodes the final latent state into the output~\citep{geiping2025scaling}. For a fully-looped model~\citep{zhu2025ouro}, both the $\operatorname{Prelude}$ and $\operatorname{Coda}$ degenerate to the identity. In this work, we focus on the middle-loop architecture, which retains separate prelude and coda modules. Additional related work is discussed in Appendix~\ref{apx.extended.related.work}.

\section{A Unified Framework for Looped Transformers}
\label{sec:unified-framework}

This section develops a unified framework for Looped Transformers by focusing on the recurrent state transitions induced by shared recurrent blocks. We first characterize the transition structure of representative architectures, then formulate loop transitions from an optimization perspective, and finally map these transitions into the resulting framework, as summarized in Table~\ref{tab:loop-optimizer-view}.

\subsection{A Transition View of Looped Transformers}
\label{sec:transition-centric-view}


\paragraph{Vanilla State Transition.} The vanilla loop updates the recurrent state through a simple transition. Let $\boldsymbol{y}_l\in \mathbb{R}^d$ denote the recurrent state and $\boldsymbol{b}\in \mathbb{R}^d$ an optional bias. Its transition is given by Eq.~\ref{eq:prelim-loop-recurrence}. Given the prelude output $\boldsymbol{e}=\operatorname{Prelude}_{\phi}(\mathbf{H}^{(0)})\in \mathbb{R}^d$, the initial state is either $\boldsymbol{y}_0=\boldsymbol{e}$ or $\boldsymbol{y}_0\sim \mathcal{N}(\mu,\sigma)$. Under this abstraction, Ouro~\citep{zhu2025ouro} omits bias ($\boldsymbol{b}=\mathbf{0}$), whereas Huginn~\citep{geiping2025scaling} uses injection ($\boldsymbol{b}=\boldsymbol{e}$). The final state is read out as $\mathbf{H}^{(L)}=\operatorname{Coda}_{\phi}(\boldsymbol{y}_R)$.
\begin{equation}
    \boldsymbol{y}_{l+1}=\operatorname{Blocks}_{\theta}(\boldsymbol{y}_{l}+\boldsymbol{b}).
    \label{eq:prelim-loop-recurrence}
\end{equation}
\paragraph{Structured State Transition.} Parcae~\citep{prairie2026parcae} augments the vanilla transition with a structured state space model inspired by~\citet{gu2021efficiently}. Its transition\footnote{The original paper writes the recurrent update as $\boldsymbol{y}_{l+1}=\bar{\mathbf{A}}\boldsymbol{y}_l+\bar{\mathbf{B}}\boldsymbol{e}+\bar{\mathcal{R}}(\boldsymbol{y}_l,\boldsymbol{e})$, which represents the computation inside the shared blocks. We abstract the shared blocks as $\operatorname{Blocks}_{\theta}$. Here $\Delta,\mathbf A >0$, $\bar{\mathbf C}\in\mathbb{R}^{d\times d}$.} is given by Eq.~\ref{eq:prelim-parcae-map}. For learnable parameters $\Delta$, $\mathbf{A}\in\mathbb{R}^{d}$ and $\mathbf{B}\in\mathbb{R}^{d\times d}$, the diagonal operator $\bar{\mathbf{A}} =\operatorname{Diag}(\exp(- \Delta\mathbf{A}))$ decays the state $\boldsymbol{y}_l\in \mathbb{R}^d$ and promotes $\rho(\bar{\mathbf{A}})<1$ for stable dynamics.
Following~\citet{huang2026exact}, the input gain is $\bar{\mathbf{B}}=\operatorname{Diag}(\frac{1-\exp(-\Delta\mathbf{A})}{\mathbf{A}})\mathbf{B}$. The final state is read out as $\mathbf{H}^{(L)}=\operatorname{Coda}_{\phi}(\bar{\mathbf{C}}\ \boldsymbol{y}_R)$.
\begin{equation}
    \boldsymbol{y}_{l+1}=\operatorname{Blocks}_{\theta}(\bar{\mathbf{A}}\boldsymbol{y}_{l}+\bar{\mathbf{B}}\boldsymbol{e}).
    \label{eq:prelim-parcae-map}
\end{equation}
\paragraph{Expanded State Transition.} HyperLoop~\citep{zeitoun2026hyperloop} expands the recurrent state to $r$ streams as $\mathbf{Y}_{l+1}\in \mathbb{R}^{r\times d}$. Its transition is given by Eq.~\ref{eq:prelim-hyperloop-recurrence}. The state-dependent transformations $\mathbf{H}^{\mathrm{pre}}_l,\mathbf{H}^{\mathrm{post}}_l\in\mathbb{R}^{1\times r}$ map between the streams and the $d$-dimensional block input, while the matrix $\mathbf{H}^{\mathrm{res}}_l\in\mathbb{R}^{r\times r}$ mixes the carried state. The final state is read out as $\mathbf{H}^{(L)}=\operatorname{Coda}_{\phi}(\operatorname{Mean}_r(\mathbf{Y}_R))$. Here $\boldsymbol{b}_l\in\mathbb{R}^{d}$ is an optional independent bias, which we omit for notational simplicity.
\begin{equation}
    \mathbf{Y}_{l+1}
    =
    \mathbf{H}^{\mathrm{res}}_l\mathbf{Y}_l
    +
    (\mathbf{H}^{\mathrm{post}}_l)^{\top}
    \left( \operatorname{Blocks}_{\theta} \left(
        \mathbf{H}^{\mathrm{pre}}_l\mathbf{Y}_l
      \right)
      +\mathbf{b}_l \right).
    \label{eq:prelim-hyperloop-recurrence}
\end{equation}

Despite their architectural differences, these models share a common pattern: they read the recurrent state, transform it using the shared block $\operatorname{Blocks}_{\theta}$ and write the result back as the next state.

\begin{table}[t!]
\caption{Comparison of loop transitions under the optimizer update rule in Eq.~\ref{eq:sgd-optimizer}. The lower rows identify the decay coefficient ($\lambda$), step size ($\eta$), input map ($\boldsymbol{x}$), output error ($f^{\prime}$), and recurrent state ($\mathbf{W}$). Red marks mismatches between the update's implied input map and the projection's map. Aligning the transition with the input map used in the projection improves performance (§~\ref{sec.align.input.map}).}
\label{tab:loop-optimizer-view}
\centering
\resizebox{\linewidth}{!}{%
\begin{tabular}{c|cccc}
\toprule
\textbf{Models}
& \begin{tabular}[c]{c}
  \textbf{Ouro}\\
  \citep{zhu2025ouro}\\
  \end{tabular}
  & \begin{tabular}[c]{c}
  \textbf{Huginn}\\
  \citep{geiping2025scaling}
  \end{tabular}
& \begin{tabular}[c]{c}
  \textbf{Parcae}\\
  \citep{prairie2026parcae}
  \end{tabular}
& \begin{tabular}[c]{c}
  \textbf{HyperLoop}\\
  \citep{zeitoun2026hyperloop}
  \end{tabular}\\
\midrule
\multicolumn{4}{l}{\textit{Original loop transitions}}\\
\midrule
\textbf{Loop}
& $\begin{aligned}
  \boldsymbol{y}_{l+1}=\operatorname{Blocks}_{\theta}(\boldsymbol{o}_l)
  \end{aligned}$
  & $\begin{aligned}
  \boldsymbol{y}_{l+1}=\operatorname{Blocks}_{\theta}(\boldsymbol{o}_l)
  \end{aligned}$
& $\begin{aligned}
  &\boldsymbol{y}_{l+1}=\operatorname{Blocks}_{\theta}(\boldsymbol{o}_l)
  \end{aligned}$
& $\begin{aligned}
  &\mathbf{Y}_{l+1}=\mathbf{H}^{\mathrm{res}}_l\mathbf{Y}_l+\\
  &(\mathbf{H}^{\mathrm{post}}_l)^{\top}\operatorname{Blocks}_{\theta}(\boldsymbol{o}_l)
  \end{aligned}$\\
Projection
&
$\boldsymbol{o}_l=\boldsymbol{y}_l$ 
&
$\boldsymbol{o}_l=\boldsymbol{y}_l+\mathbf{e} $
&
$\boldsymbol{o}_l=\bar{\mathbf{A}}\boldsymbol{y}_l+\bar{\mathbf{B}}\mathbf{e}$
&
$\boldsymbol{o}_l=\mathbf{H}^{\mathrm{pre}}_l\mathbf{Y}_l$
\\
\midrule
\multicolumn{4}{l}{\textit{Optimization mapping to  } $\mathbf{W}_{i+1}=(1-\lambda \eta)\ \mathbf{W}_{i} - \eta \  (\boldsymbol{x}^{\top}f^{\prime})$}\\
\midrule
$\lambda$
& $1$
& $1$
& $1$
& $\mathbf{I}-\mathbf{H}^{\mathrm{res}}_l$\\
$\eta$
& $1$
& $1$
& $1$
& $1$\\
$\boldsymbol{x}$
& $\mathbf{I}$
& $\mathbf{I}$
& ${\color{darkred}\mathbf{I} (\text{Mismatch with }\bar{\mathbf{A}})}$
& ${\color{darkred}\mathbf{H}^{\mathrm{post}}_l (\text{Mismatch with } \mathbf{H}^{\mathrm{pre}}_l)}$\\
$f^{\prime}(\boldsymbol{o}_l)$
& $-\operatorname{Blocks}_{\theta}(\boldsymbol{o}_l)$
& $-\operatorname{Blocks}_{\theta}(\boldsymbol{o}_l)$
& $-\operatorname{Blocks}_{\theta}(\boldsymbol{o}_l)$
& $-\operatorname{Blocks}_{\theta}(\boldsymbol{o}_l)$\\
$\mathbf{W}$
& $\boldsymbol{y}_l\in\mathbb{R}^{d}$
& $\boldsymbol{y}_l\in\mathbb{R}^{d}$
& $\boldsymbol{y}_l\in\mathbb{R}^{d}$
& $\mathbf{Y}_l\in\mathbb{R}^{r\times d}$\\
\bottomrule
\end{tabular}
}
\end{table}

\subsection{An Optimization View of Looped Transformers}
\label{sec:unified-optimizer-view}
To build a connection between the looped transition and optimization, we interpret the recurrent state as a fast weight updated across the depth. Then we further express loop transitions in an optimization form. This formulation reveals the implicit target predicted by the recurrent blocks.

We first consider a toy optimization example. Let
$\boldsymbol{x}\in \mathbb{R}^{1 \times d_1}$, $\boldsymbol{b} \in \mathbb{R}^{1 \times d_2}$ and $\mathbf{W}\in \mathbb{R}^{d_1\times d_2}$ be an input, bias and weights, respectively. The output is $\boldsymbol{o}=\boldsymbol{x}\mathbf{W}+\boldsymbol{b} \, \in \mathbb{R}^{1\times d_2}$. For the local objective $ \mathcal{L}:=f(\boldsymbol{o})$, the output error is $\frac{\partial \mathcal{L}}{\partial \boldsymbol{o}} := f^{\prime}(\boldsymbol{o}) \, \in \mathbb{R}^{1\times d_2}$. Then the gradient with respect to the weight is $\frac{\partial \mathcal{L}}{ \partial \mathbf{W}} =\boldsymbol{x}^{\top}f^{\prime}(\boldsymbol{o}) \,  \in \mathbb{R}^{d_1 \times d_2}$. For step size $\eta$ and weight decay $\lambda$, the optimization update is:
\begin{equation}\label{eq:sgd-optimizer}
    \mathbf{W}_{i+1} = \mathbf{W}_{i} - \eta \,(\nabla_{\mathbf{W}}\mathcal{L} + \lambda \,\mathbf{W}_{i} )=
    (1-\eta\lambda)\,\mathbf{W}_{i}
    -\eta\, (\boldsymbol{x}^{\top}f^{\prime}(\boldsymbol{o})) .
\end{equation}
\textit{To relate loop transitions to optimization, we interpret the recurrent state $\mathbf{Y}_l\in\mathbb{R}^{d_1\times d_2}$ as a fast weight playing the role of $\mathbf{W}$ in Eq.~\ref{eq:sgd-optimizer}.} The dimensions $d_1,d_2$ denote its row and column dimensions, which are specified for each architecture. Proposition~\ref{proposition} gives the resulting optimizer-form transition.

\begin{proposition}[Optimization interpretation of loop transitions]\label{proposition}
Let $\mathbf{Y}_l\in\mathbb{R}^{d_1\times d_2}$ denote the fast weight updated across recurrent depth $l$. Let $\mathbf{H}^{\mathrm{in}}_l\in\mathbb{R}^{m\times d_1}$ be an adaptive input map and $\boldsymbol{b}_l\in\mathbb{R}^{m\times d_2}$ a bias. For a projection $\boldsymbol{o}_l=\mathbf{H}^{\mathrm{in}}_l\mathbf{Y}_l+\boldsymbol{b}_l$ with the local objective $\mathcal{L}:=f(\boldsymbol{o}_l)$, the loop transition can be written in the gradient-based optimization form shown in Eq.~\ref{eq:unified-looped-transition}:
\begin{equation}\label{eq:unified-looped-transition}
    \mathbf{Y}_{l+1}
=\left(\mathbf{I}_{d_1}-\eta_l\mathbf{\Lambda}_l\right)\mathbf{Y}_l
-\eta_l \left(  (\mathbf{H}^{\mathrm{in}}_l)^{\top} f^{\prime}(\boldsymbol{o}_l) \right) ,
\end{equation}
where $\eta_l$ is an adaptive step size and $\mathbf{\Lambda}_l\in\mathbb{R}^{d_1\times d_1}$ is a decay operator representing weight decay.
\end{proposition}
\vspace{-0.8em}
\begin{proof}[Proof Sketch]
From the optimizer perspective, the general update in Eq.~\ref{eq:unified-looped-transition} contains an implicit forward computation and a local objective evaluation. In particular, the fast weight projection is computed as $\boldsymbol{o}_l = \mathbf{H}_l^{\mathrm{in}}\mathbf{Y}_l + \boldsymbol{b}_l \,\in \mathbb{R}^{m\times d_2}$, where $\mathbf{H}_l^{\mathrm{in}} \in \mathbb{R}^{m\times d_1}$ is the input map and $\boldsymbol{b}_l\in\mathbb{R}^{m\times d_2}$ is an optional bias. 
To update the fast weight $\mathbf{Y}_l$, we consider gradient-based optimization with decay:
\begin{equation}\label{eq:fast-weight-optimiation}
\mathbf{Y}_{l+1} = \mathbf{Y}_{l}  - \eta_l \left( \nabla_{\mathbf{Y}_l}\mathcal{L} + \mathbf{\Lambda}_l \mathbf{Y}_{l} \right) .
\end{equation}
The partial derivative of $ \mathcal{L}$ is $ \frac{\partial \mathcal{L}}{\partial \boldsymbol{o}_l} =f^{\prime}(\boldsymbol{o}_l)\in \mathbb{R}^{m \times d_2}$.
We then obtain the gradient:\begin{equation}\label{eq:graddients-general-transition}
\nabla_{\mathbf{Y}_l}\mathcal{L} = \frac{\partial \mathcal{L}}{\partial \mathbf{Y}_l} = (\mathbf{H}^{\mathrm{in}}_l)^{\top}  \frac{\partial \mathcal{L}}{\partial  \boldsymbol{o}_l}= (\mathbf{H}^{\mathrm{in}}_l)^{\top} f^{\prime}(\boldsymbol{o}_l) .
\end{equation}
The gradient has shape $d_1\times d_2$, matching that of the fast weight. The matrix product in Eq.~\ref{eq:graddients-general-transition} reduces to an outer product when $m=1$. Substituting Eq.~\ref{eq:graddients-general-transition} into Eq.~\ref{eq:fast-weight-optimiation} yields:
\begin{equation}\label{eq:final-derivation}
\mathbf{Y}_{l+1} = \mathbf{Y}_{l} - \eta_l \left( (\mathbf{H}^{\mathrm{in}}_l)^{\top} f^{\prime}(\boldsymbol{o}_l)+ \mathbf{\Lambda}_l \mathbf{Y}_{l} \right) = \left(\mathbf{I}_{d_1}-\eta_l\mathbf{\Lambda}_l\right)\mathbf{Y}_l
-\eta_l \left(  (\mathbf{H}^{\mathrm{in}}_l)^{\top} f^{\prime}(\boldsymbol{o}_l) \right).
\end{equation}
Finally, Eq.~\ref{eq:final-derivation} recovers the formulation in Eq.~\ref{eq:unified-looped-transition}, as stated in Proposition~\ref{proposition}.
\end{proof}

Let $\boldsymbol{t}_l=\operatorname{Blocks}_{\theta}(\boldsymbol{o}_l)$ be the implicit target predicted by the shared recurrent blocks. For the local algebraic derivation, we hold $\boldsymbol{t}_l$ fixed when differentiating the implicit negative inner product objective $\mathcal{L}_l=-\langle\boldsymbol{o}_l,\boldsymbol{t}_l\rangle$, giving $f^{\prime}(\boldsymbol{o}_l)=\nabla_{\boldsymbol{o}_l}\mathcal{L}_l=-\boldsymbol{t}_l =-\operatorname{Blocks}_{\theta}(\boldsymbol{o}_l)$. The fixed-target convention is used only to derive the closed-form transition algebraically. It does not imply that gradients are stopped during training. Under this hypothesis and Proposition~\ref{proposition}, we can map the transition in §~\ref{sec:existing-transitions-coordinates}.


\vspace{-0.1cm}
\subsection{Mapping Loop Transitions}
\label{sec:existing-transitions-coordinates}
\vspace{-0.1cm}
In this subsection, we map the loop transitions to the optimizer update rule introduced in §~\ref{sec:unified-optimizer-view}. We identify the recurrent state, projection, input map, decay operator and step size for Vanilla Loop, Parcae and HyperLoop. This analysis reveals input-map mismatches in Parcae and HyperLoop. The resulting mappings are summarized in Table~\ref{tab:loop-optimizer-view}.

\vspace{-0.2cm}
\paragraph{Vanilla Loop.}
Let $\mathbf{Y}_l=\boldsymbol{y}_l\in\mathbb{R}^{d\times1}$ and $\boldsymbol{b}\in\mathbb{R}^{d\times1}$. The projection $\boldsymbol{o}_l=\boldsymbol{y}_l+\boldsymbol{b}$ uses $\mathbf{H}^{\mathrm{in}}_l=\mathbf{I}_d$ and $\boldsymbol{b}_l=\boldsymbol{b}$. Setting $\eta_l=1$ and $\mathbf{\Lambda}_l=\mathbf{I}_d$ in Eq.~\ref{eq:unified-looped-transition} gives $\boldsymbol{y}_{l+1}=-f^{\prime}(\boldsymbol{o}_l)=\operatorname{Blocks}_{\theta}(\boldsymbol{o}_l)=\operatorname{Blocks}_{\theta}(\boldsymbol{y}_l+\boldsymbol{b})$, which recovers Eq.~\ref{eq:prelim-loop-recurrence}. Vanilla Loop is a special case of the optimizer update rule.

\vspace{-0.2cm}
\paragraph{Parcae.}
For the structured state with $\boldsymbol{e}\in\mathbb{R}^{d\times1}$, the projection $\boldsymbol{o}_l=\bar{\mathbf{A}}\boldsymbol{y}_l+\bar{\mathbf{B}}\boldsymbol{e}$ uses $\mathbf{H}^{\mathrm{in}}_l=\bar{\mathbf{A}}\in\mathbb{R}^{d\times d}$ and $\boldsymbol{b}_l=\bar{\mathbf{B}}\boldsymbol{e}\in\mathbb{R}^{d\times1}$. Setting $\eta_l=1$ and $\mathbf{\Lambda}_l=\mathbf{I}_d$ in Eq.~\ref{eq:unified-looped-transition} yields $\boldsymbol{y}_{l+1}= -{\bar{\mathbf{A}}}^{\top} f^{\prime}(\boldsymbol{o}_l)={\bar{\mathbf{A}}}^{\top}\operatorname{Blocks}_{\theta}(\boldsymbol{o}_l)$. Compared with Eq.~\ref{eq:prelim-parcae-map}, this update exposes the missing input-map factor.

\vspace{-0.2cm}
\paragraph{HyperLoop.}
For the expanded state $\mathbf{Y}_l\in\mathbb{R}^{r\times d}$, the projection $\boldsymbol{o}_l=\mathbf{H}_l^{\mathrm{pre}}\mathbf{Y}_l\in\mathbb{R}^{1\times d}$ uses $\mathbf{H}^{\mathrm{in}}_l=\mathbf{H}_l^{\mathrm{pre}}\in\mathbb{R}^{1\times r}$ without bias. Setting $\eta_l=1$ and $\mathbf{\Lambda}_l=\mathbf{I}_r-\mathbf{H}^{\rm res}_l$ in Eq.~\ref{eq:unified-looped-transition} gives $\mathbf{Y}_{l+1} =  \mathbf{H}^{\rm res}_l \mathbf{Y}_{l} +{\mathbf{H}^{\rm pre}_l}^{\top} \operatorname{Blocks}_{\theta}(\mathbf{H}^{\rm pre}_l \mathbf{Y}_{l})$. Compared with Eq.~\ref{eq:prelim-hyperloop-recurrence}, this update requires input-map alignment.

After mapping, Parcae and HyperLoop exhibit input-map mismatches under the optimization view. The gradient form in Eq.~\ref{eq:graddients-general-transition} uses the matrix product of the transposed input map and the output error: Parcae omits $\bar{\mathbf A}^{\top}$, while HyperLoop uses separate $\mathbf H_l^{\mathrm{pre}}$ and $\mathbf H_l^{\mathrm{post}}$. After input-map alignment, both Parcae and HyperLoop show improved performance as shown in §~\ref{sec.align.input.map}, providing further support for the proposed framework. These findings motivate the optimizer-guided loop design in §~\ref{sec:optimizer-guided-loop-design}.

\vspace{-0.1cm}
\section{Optimizer-Guided Loop Design} \label{sec:optimizer-guided-loop-design}
\vspace{-0.1cm}
To examine whether the framework can guide loop-transition design, we derive OperLoop using HyperLoop's expanded state and a delta objective. Let
$\mathbf{Y}_l\in\mathbb{R}^{r\times d}$ be the expanded fast weight and $\mathbf{H}^{\mathrm{in}}_l\in\mathbb{R}^{1\times r}$ the input map. The projection is $\boldsymbol{o}_l=\mathbf{H}^{\mathrm{in}}_l\mathbf{Y}_l \in \mathbb{R}^{1 \times d}$. We switch to the delta objective $\mathcal{L}= \frac{1}{2} \left\|\boldsymbol{o}_l- \boldsymbol{t}_l \right\|_2^2$, whose gradient with respect to the fast weight is $\nabla_{\mathbf{Y}_l}\mathcal{L}_l=(\mathbf{H}^{\mathrm{in}}_l)^{\top}(\boldsymbol{o}_l-\boldsymbol{t}_l)=(\mathbf{H}^{\mathrm{in}}_l)^{\top}(\boldsymbol{o}_l - \operatorname{Blocks}_{\theta}(\boldsymbol{o}_l))$\footnote{This delta objective is also inspired by Delta Linear Attention~\citep{yang2024parallelizing,yang2025gated,kimiteam2025kimilinearexpressiveefficient,huang2026mdn}. The main distinction is that the looped block implicitly predicts the target representation, whereas Delta-based Linear Attention explicitly generates the value vector through linear projections.}. Substituting this gradient into Eq.~\ref{eq:unified-looped-transition} yields:
\begin{align}
\mathbf{Y}_{l+1} = (\mathbf{I} - \eta_l \mathbf{\Lambda}_l) \ \mathbf{Y}_l +\eta_l \ (\mathbf{H}^{\mathrm{in}}_l)^{\top}
 \left(\operatorname{Blocks}_{\theta}(\boldsymbol{o}_l) - \boldsymbol{o}_l \right) \label{eq.pre.Y_l+1} .
\end{align}
We parameterize the state-dependent input map $\mathbf{H}^{\mathrm{in}}_l$, decay operator $\mathbf{\Lambda}_l$ and step size $\eta_l$ as follows:
\begin{align}
\eta_l &= \sigma\left(
a^{\mathrm{lr}}
\cdot
\left(
\mathbf{W}^{\mathrm{lr}} \mathbf{Z}_l
\right)
+
b^{\mathrm{lr}}
\right) \cdot \eta_{l-1} &\in \mathbb{R}, \\
\mathbf{\Lambda}_l
&=
\operatorname{Diag}
\left(
\sigma\left(
\boldsymbol{a}^{\mathrm{wd}}
\cdot
\left(
\mathbf{W}^{\mathrm{wd}} \mathbf{Z}_l
\right)
+
\boldsymbol{b}^{\mathrm{wd}}
\right)
\right) & \in \mathbb{R}^{r\times r}, \\
\mathbf{H}^{\mathrm{in}}_{l}
&=  
\sigma\left(
\boldsymbol{a}^{\mathrm{in}}
\cdot
\left(
\mathbf{W}^{\mathrm{in}} \mathbf{Z}_l
\right)
+
\boldsymbol{b}^{\mathrm{in}}
\right) &\in \mathbb{R}^{1 \times r}, 
\end{align}
where the causal step size decreases monotonically with $\eta_l = \sigma \cdot \eta_{l-1}$ and is initialized with $\eta_{-1} = 1$. This design is inspired by learning-rate scheduling.
Here, $\sigma$ denotes the sigmoid function. We set $\mathbf{Z}_l=\operatorname{RMSNorm}(\operatorname{flatten}(\mathbf{Y}_l))\in\mathbb{R}^{rd}$. The projection matrices satisfy $\mathbf{W}^{\mathrm{wd}}, \mathbf{W}^{\mathrm{in}}\in \mathbb{R}^{r\times rd}$, while $\mathbf{W}^{\mathrm{lr}} \in \mathbb{R}^{1\times rd}$. The biases satisfy $\boldsymbol{b}^{\mathrm{wd}},\boldsymbol{b}^{\mathrm{in}}\in\mathbb{R}^{r}$ and $\boldsymbol{b}^{\mathrm{lr}}\in\mathbb{R}$. After $R$ loops, we obtain $\mathbf{Y}^{(R)}$.

This architecture can also be expressed in a transition form similar to that of HyperLoop, but with a different parameterization. Substituting $\boldsymbol{o}_l=\mathbf{H}^{\mathrm{in}}_l\mathbf{Y}_l$ into Eq.~\ref{eq.pre.Y_l+1}, we get the new transition:
\begin{align}
 \mathbf{Y}_{l+1}  &= \left(\mathbf{I} - \eta_l (\mathbf{\Lambda}_l + (\mathbf{H}^{\mathrm{in}}_l)^{\top}\mathbf{H}^{\mathrm{in}}_l)\right)\mathbf{Y}_l +\eta_l(\mathbf{H}^{\mathrm{in}}_l)^{\top}
 \operatorname{Blocks}_{\theta}(\mathbf{H}^{\mathrm{in}}_l\mathbf{Y}_l) \label{eq.after.Y_l+1} .
\end{align} 
Specifically, setting $\mathbf{H}^{\mathrm{pre}}=\mathbf{H}^{\mathrm{in}}$, $\mathbf{H}^{\mathrm{post}}=\eta_l\mathbf{H}^{\mathrm{in}}$, and $\mathbf{H}^{\mathrm{res}}=\mathbf{I}-\eta_l\left(\mathbf{\Lambda}_l+(\mathbf{H}^{\mathrm{in}}_l)^{\top}\mathbf{H}^{\mathrm{in}}_l\right)$ recovers the proposed transition. This correspondence allows us to implement this loop transition within the HyperLoop implementation, as shown in Algorithm~\ref{alg:hyperloop-delta-causal-lr} in Appendix~\ref{apx.pseudocode}.

\section{Experiments} 
\label{sec:experiments}

\vspace{-0.2cm}

We use an MoE backbone with sliding-window attention (SWA) and full attention in a 3:1 ratio~\citep{huang2026step35flashopen}. Our non-looped baselines are an 18-layer, 5B-parameter model with 600M activated parameters and a depth-scaled 32-layer, 10B-parameter model with 900M activated parameters. Following HyperLoop~\citep{zeitoun2026hyperloop}, each looped variant uses the middle-loop pattern and matches the unrolled training FLOPs of its non-looped baseline, with approximately 50\% as many parameters. All HyperLoop and OperLoop variants use four parallel streams ($r=4$) throughout the experiments. After unrolling, the looped and non-looped baseline models have the same effective depth, KV-cache layout, and attention pattern. Comparisons among looped variants under these matched settings isolate the effect of the loop transition.

All models are trained with a sequence length of $4096$. The 18-layer 5B and 32-layer 10B baselines are trained on approximately $300$B and $500$B tokens over $67$K and $120$K updates, respectively. Both scales use cosine learning-rate decay, a weight decay of $0.1$, and global gradient clipping at $1.0$. Model parameters are optimized with Muon using momentum $0.95$ and six Polar Express NS iterations~\citep{liu2025muon}. The looped variants follow the same training recipe as their corresponding baselines. We evaluate perplexity and downstream accuracy using the open-source \texttt{lm-eval-harness} toolkit~\citep{gao2023framework}, largely following its default settings.
Further model and training details are provided in Appendix~\ref{app:configuration}, and evaluation details are given in Appendix~\ref{app:evaluation}.

\vspace{-0.2cm}
\subsection{Controlled Study of Optimizer-Aligned Loop Transitions}\label{sec.align.input.map}
\begin{figure}[t!]
    \centering
    \includegraphics[width=\linewidth]{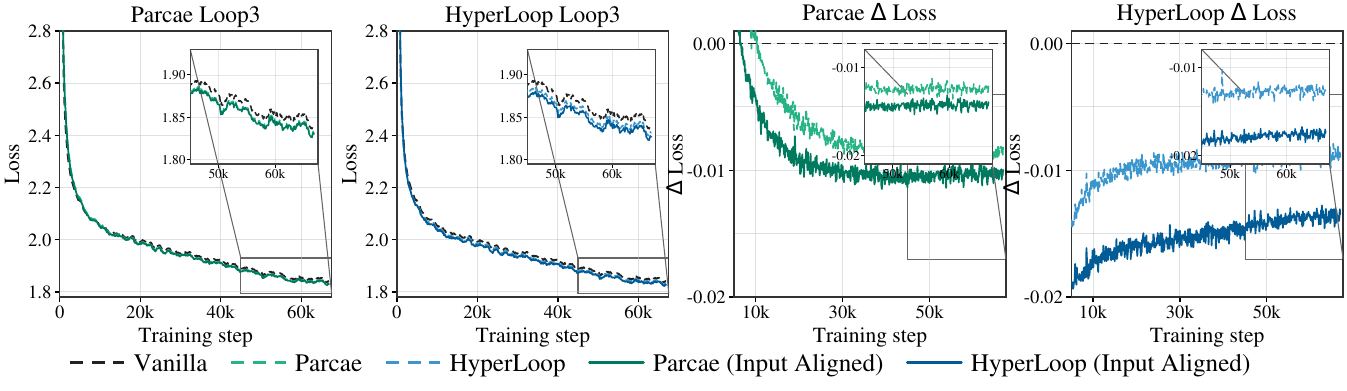}
    \vspace{-0.6cm}
    \caption{Training-loss comparison for optimizer-aligned loop transitions. The left two panels show training-loss trajectories for Parcae Loop3 and HyperLoop Loop3, together with their input-map-aligned variants. The right two panels show the corresponding loss differences relative to Vanilla Loop3, defined as $\Delta\mathcal{L}=\mathcal{L}_{\mathrm{variant}}-\mathcal{L}_{\mathrm{Vanilla}}$. Negative values indicate lower loss than Vanilla Loop3. 
    } \label{fig:loss-curves}
    \vspace{-0.2cm}
\end{figure}

\begin{table}[t!]
\centering
\caption{Effect of input-map alignment on Parcae and HyperLoop in language modeling and commonsense evaluation. The looped variants differ only in their transitions. Parameters and training FLOPs are relative to the 18-layer baseline. The subscript \(n\) denotes normalized accuracy. $\Delta$Avg. is the improvement over the corresponding unaligned model, in percentage points.
}
\label{tab:input-map-alignment}
\renewcommand{\arraystretch}{1.15}
\setlength{\tabcolsep}{3.6pt}
\resizebox{\linewidth}{!}{%
\begin{tabular}{lcccccccccccc}
\toprule
\multirow{2}{*}{\textbf{Model}}
&
\multirow{2}{*}{\textit{Param.}}
&
\multirow{2}{*}{\textit{FLOPs}}
&  \multicolumn{2}{c}{\textbf{PPL} $\downarrow$}& \multicolumn{7}{c}{\textbf{Commonsense Acc. (\%)} $\uparrow$}\\
\cmidrule(lr){4-5}\cmidrule(lr){6-12}
& & &  Lamb. &Wiki. & Lamb.
& ARCc$_n$ & ARCe$_n$
& HellaS.$_n$ & WinoG. & PIQA$_n$ & Avg. & $\Delta$Avg. \\
\midrule
\textbf{Baseline} $_\texttt{[18L]}$
& $1\times$
& $1\times$
&  6.56 &13.08 & 60.16 & 40.96 & 71.13 & 64.78 & 59.51 & 76.61 & 62.19 & --\\
\midrule
\textbf{Parcae} $_\texttt{[4L-4L×3-2L]}$
& $0.56\times$
&$1\times$
&  7.71 &13.98 & 57.44 & \textbf{38.57} & \textbf{65.91} & 62.00 & 56.12 & 74.81 & 59.14 & --\\
\rowcolor{orange!8} \textbf{Parcae} (Aligned)
& $0.56\times$
&$1\times$
&  \textbf{7.41} &\textbf{13.93} & \textbf{58.59} & 37.71 & 65.53 & \textbf{62.09} & \textbf{56.91} & \textbf{74.92} & \textbf{59.29} & \textbf{+0.15}\\
\midrule
\textbf{HyperLoop} $_\texttt{[4L-4L×3-2L]}$
& $0.56\times$
&$1\times$
&  7.68 &13.96 & 57.46 & \textbf{38.40} & 63.51 & 61.41 & 57.77 & \textbf{75.03} & 58.93 & --\\
\rowcolor{orange!8} \textbf{HyperLoop} (Aligned)
& $0.56\times$
&$1\times$
&  \textbf{7.35} &\textbf{13.93} & \textbf{58.63} & 35.15 & \textbf{66.67} & \textbf{62.83} & \textbf{58.56} & 74.97 & \textbf{59.47} & \textbf{+0.54}\\
\bottomrule
\end{tabular}%
}
\end{table}

We first test whether the framework can guide improvements to existing loop transitions. Specifically, input-map alignment enforces the transpose relationship between the read and write maps required by the corresponding gradient update. For Parcae, we replace
Eq.~\ref{eq:prelim-parcae-map} with the aligned update
$\boldsymbol{y}_{l+1}=\bar{\mathbf{A}}^{\top}
\operatorname{Blocks}_{\theta}(\bar{\mathbf{A}}\boldsymbol{y}_{l}+\bar{\mathbf{B}}\boldsymbol{e})$. For HyperLoop in Eq.~\ref{eq:prelim-hyperloop-recurrence}, we set
$\mathbf{H}^{\mathrm{post}}_l=\mathbf{H}^{\mathrm{pre}}_l$, giving
$\mathbf{Y}_{l+1}=\mathbf{H}^{\mathrm{res}}_l\mathbf{Y}_l+
(\mathbf{H}^{\mathrm{pre}}_l)^{\top}(\operatorname{Blocks}_{\theta}(\mathbf{H}^{\mathrm{pre}}_l\mathbf{Y}_l)+\mathbf{b}_l)$. This constraint also reduces the number of parameters in the HyperLoop transition by removing the separate $\mathbf{H}^{\mathrm{post}}$ map.

Alignment improves training loss, perplexity, and average commonsense accuracy for both Parcae and HyperLoop (Figure~\ref{fig:loss-curves}; Table~\ref{tab:input-map-alignment}). For example, HyperLoop improves average accuracy from 58.93 to 59.47 while using fewer transition parameters. \textbf{These results support input-map alignment as a useful design constraint and provide empirical support for the optimization view}.


\subsection{Main Downstream Results}
\label{subsec:main_results}
\begin{table*}[t!]
\centering
\caption{Main results on commonsense, math, code, and reasoning benchmarks. Parameters and training FLOPs are relative to the corresponding non-looped baseline. \textbf{Bold} and \underline{underlined} values indicate the best and second-best results among looped models at each scale, respectively. The subscript $n$ denotes normalized accuracy, and $*$ denotes CoT evaluation. H.Eval and H.Eval$^+$ denote HumanEval and HumanEval+, respectively, both report pass@1 with greedy decoding. 
}
\scriptsize
\setlength{\tabcolsep}{2.6pt}

\renewcommand{\arraystretch}{1.1}
\resizebox{\textwidth}{!}{%
\begin{tabular}{@{}llccc>{\columncolor{orange!8}}cccc>{\columncolor{orange!8}}c@{}}
\toprule

& & \multicolumn{4}{c}{\textit{300B Training Tokens, 4k Sequence Length}} & \multicolumn{4}{c}{\textit{500B Training Tokens, 4k Sequence Length}} \\

\cmidrule(lr){3-6}
\cmidrule(lr){7-10}

\textbf{Category}
& \textbf{Dataset}
& \textbf{Base}$_\texttt{[18L]}$
& \multicolumn{3}{c}{\textbf{Loop}$_\texttt{[4L-4L×3-2L]}$}
&  \textbf{Base}$_\texttt{[32L]}$& \multicolumn{3}{c}{\textbf{Loop}$_\texttt{[4L-8L×3-4L]}$}\\

\cmidrule{4-6}
\cmidrule{8-10}

(\textit{Metrics})
&
& \textbf{5B MoE}
& \textbf{Parcae}
& \textbf{HyperLoop}
& \textbf{OperLoop}
& \textbf{10B MoE}
&  \textbf{Parcae}&\textbf{HyperLoop}
& \textbf{OperLoop} \\

\midrule

&
\textit{Parameters}& $1\times$
& $0.56\times$
& $0.56\times$
& $0.56\times$
&  $1\times$&  $0.48\times$&$0.48\times$&  $0.48\times$\\

&
\textit{FLOPs}& $1\times$
& $1\times$
& $1\times$
& $1\times$
&  $1\times$& $1\times$&$1\times$& $1\times$\\

\midrule

\multirow{2}{*}{(\textit{PPL $\downarrow$})}
& \textbf{Wikitext}
& 13.08 & 13.98 & \underline{13.96} & \textbf{13.90}
& 11.06 &  \underline{11.76} & 11.87 &  \textbf{11.74} \\
& \textbf{Lambada}& 6.56 & 7.71 & \underline{7.68} & \textbf{7.21}
& 4.82 & \underline{5.20} & 5.32 & \textbf{5.14} \\

\midrule

\multirow{7}{*}{\textit{Com.Sen.}} & \textbf{Lambada}& 60.16 & 57.44 & \underline{57.46} & \textbf{59.29}
& 67.48 & \underline{65.96} & 64.91 & \textbf{66.04} \\

& \textbf{ARCc}$_{ \rm{ n}}$& 40.96 & \textbf{38.57} & \underline{38.40} & 37.63
& 44.62 & \textbf{45.73} & \underline{45.56} & 44.34 \\

& \textbf{ARCe}$_{ \rm{ n}}$& 71.13 & \underline{65.91} & 63.51 & \textbf{65.99}
& 72.26 & 71.17 & \underline{73.23} & \textbf{73.65} \\

& \textbf{HellaSwag}$_{ \rm{ n}}$& 64.78 & \underline{62.00} & 61.41 & \textbf{62.20}
& 71.78 & \underline{70.05} & 69.33 & \textbf{70.07} \\

& \textbf{WinoGrande}& 59.51 & 56.12 & \underline{57.77} & \textbf{57.85}
& 63.61 & 62.98 & \textbf{64.17} & \underline{63.43} \\

& \textbf{PIQA}$_{ \rm{ n}}$& 76.61 & 74.81 & \underline{75.03} & \textbf{75.63}
& 78.62 & 77.41 & \underline{77.42} & \textbf{77.88} \\

\cmidrule{2-10}
(\textit{LL} $\uparrow$) & \textit{LL Avg.} & \textit{62.19} & \underline{\textit{59.14}} & \textit{58.93} & \textbf{\textit{59.77}} & \textit{66.40} & \textit{65.55} & \underline{\textit{65.77}} & \textbf{\textit{65.90}}
\\
\midrule

\multirow{2}{*}{\textit{Math}}& \textbf{GSM8K}$_{ \rm{ 5shot}}$
& 15.77
& \underline{14.25}
& 12.89
& \textbf{15.85}
&  33.06 & 32.07 & \textbf{33.36} & \underline{33.21} \\

& \textbf{GSM8K}${_{ \rm{ 8shot}}}^{*}$
& 17.44
& \underline{17.06}
& 14.03
& \textbf{17.13}
&  37.15 & \underline{39.35} & 37.68 & \textbf{39.65} \\

\multirow{2}{*}{\textit{Code}}
& \textbf{H.Eval}$_{ \rm{ 0shot}}$& 21.95
& \underline{21.34}& 18.90
& \textbf{23.78}&  28.66 & 28.05 & \textbf{30.49} & \underline{29.88} \\


& \textbf{H.Eval$^+$}$_{ \rm{ 0shot}}$
& 18.90
& \underline{17.68}
& 16.46
& \textbf{20.12}
& 25.00
& 24.39 & \underline{25.61}
& \textbf{26.22} \\

\multirow{2}{*}{\textit{Reason.}}& \textbf{MMLU}$_{ \rm{ 5shot}}$
& 37.52
& 42.46
& \underline{42.71}
& \textbf{44.41}
&  52.81 & \textbf{53.92} & 52.62 & \underline{52.76} \\
& \textbf{BBH}${_{ \rm{ 3shot}}}^{*}$ & 31.75 & \underline{31.70} & 31.05 & \textbf{32.18} & 41.86 & \underline{39.90} & 39.64 & \textbf{43.04} \\

\cmidrule{2-10}
(\textit{EM} $\uparrow$) & \textit{EM Avg.} 
& \textit{23.89}
& \underline{\textit{24.08}}
& \textit{22.67}
& \textit{\textbf{25.58}}
& \textit{36.42} & \textit{36.28} & \underline{\textit{36.57}} & \textbf{\textit{37.46}}
\\

\midrule

\textbf{Overall}
& \textit{EM \& LL Avg.}
& \textit{43.04}
& \underline{\textit{41.61}}
& \textit{40.80}
& \textbf{\textit{42.68}}
& \textit{51.41} & \textit{50.92} & \underline{\textit{51.17}} & \textbf{\textit{51.68}} \\

\bottomrule
\end{tabular}%
}
\label{tab:main-results}
\end{table*}

We next test whether the framework can guide new transition designs by evaluating OperLoop. At both model scales, OperLoop improves average generative performance over the evaluated baselines.

Relative to the looped baselines in Table~\ref{tab:main-results}, OperLoop achieves the highest overall and EM averages among the evaluated looped models at both scales. OperLoop also achieves the lowest PPL and highest LL average among the looped variants. Under matched training FLOPs, these \textbf{results support the benefits of its transition design for both language modeling and generation}.

Relative to the non-looped baselines in Table~\ref{tab:main-results}, OperLoop uses approximately half the parameters while improving the EM average by 1.69 and 1.04 percentage points at the two scales. For PPL and LL accuracy, all looped models perform slightly worse at both scales. These differing rankings indicate that PPL and LL alone would understate looped models' generative performance in this comparison. Overall, \textbf{even though its PPL and LL results are slightly worse than those of the non-looped baseline, a looped model can achieve better results in the EM evaluation}.

The gains also vary across tasks and scales. Performance on both code benchmarks improves over that of the non-looped baselines at both scales. At the larger scale, OperLoop surpasses the baseline on GSM8K with CoT and BBH, reaching 43.04 on BBH compared with 41.86. The higher EM average thus reflects different task-level gains across scales. \textbf{Together, the alignment and OperLoop results support the framework for modifying existing transitions and designing new ones}.



\newpage
\subsection{Ablation Study}\label{sec:ablations}

We first evaluate scaling from Loop 3 to Loop 6, then ablate the delta objective and step-size schedule with Loop 3. We report Wikitext and Lambada perplexity (PPL), average commonsense accuracy based on log-likelihood (LL), and the reported six-task exact match (EM) average.


\begin{wrapfigure}{r}{0.32\linewidth}
\vspace{-0.4cm}
\centering
\resizebox{\linewidth}{!}{\includegraphics{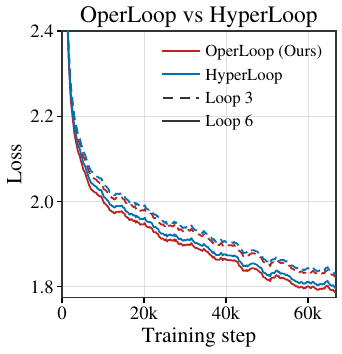}}
\vspace{-0.8cm}
\caption{Scaling loop from three (Loop3) to six (Loop6).}
\label{fig:ablation-loopscale}
\vspace{-0.32cm}
\end{wrapfigure}

\vspace{-0.2cm}
\paragraph{Compute Scaling.} We train Loop3 and Loop6 separately with three and six looped steps and evaluate each at its trained depth. Loop6 increases both training and inference compute to $1.67\times$ the baseline FLOPs (Figure~\ref{fig:ablation-loopscale}). Additional loops improve all six generation scores for both models, with the largest gains on math tasks (Table~\ref{tab:ablation-loop-count}). For OperLoop, the EM average rises from 25.58 to 31.90, with 8-shot CoT GSM8K improving from 17.13 to 31.08. PPL and LL accuracy also improve, although Loop6 still trails the non-looped baseline on Wikitext PPL and LL accuracy while surpassing it on Lambada PPL. The matched-compute Loop3 comparison further shows why generative evaluation matters: OperLoop exceeds the non-looped baseline in EM average despite worse PPL and LL accuracy. Overall, additional recurrent computation improves generative performance at a fixed parameter budget, highlighting the value of evaluating looped models beyond PPL and LL.
\begin{table}[t!]
\centering
\caption{Downstream evaluation at Loop3 and Loop6. Models are trained separately at each loop count. Com. Avg. denotes average commonsense accuracy. EM Avg. is the arithmetic mean of the six reported
math, code, and reasoning scores, and $\Delta$ is its difference from the
baseline.}
\label{tab:ablation-loop-count}
\renewcommand{\arraystretch}{1.15}
\setlength{\tabcolsep}{3.2pt}
\resizebox{\linewidth}{!}{%
\begin{tabular}{@{}lccccccccccccc@{}}
\toprule
& & & \multicolumn{2}{c}{PPL $\downarrow$} & LL Acc $\uparrow$ &\multicolumn{8}{c}{EM Acc $\uparrow$}\\
\cmidrule(lr){4-5}
\cmidrule(lr){6-6}
\cmidrule(lr){7-14}
Model & Params. & FLOPs & Wiki. & Lamb. & Com. Avg. & GSM8K & GSM8K$^{*}$ & H.Eval & H.Eval$^{+}$  & MMLU &BBH$^{*}$ & EM Avg. & $\Delta$ \\
\midrule
Baseline (5B, 18L) & $1\times$ & $1\times$ & \textbf{13.08} & 6.56 & \textbf{62.19} & 15.77 & 17.44 & 21.95 & 18.90  & 37.52 &31.75 & 23.89 & - \\
\midrule
HyperLoop (Loop 3) & $0.56\times$ & $1\times$ & 13.96 & 7.68 & 58.93 & 12.89 & 14.03 & 18.90 & 16.46  & 42.71 &31.05 & 22.67 & -1.22 \\
\rowcolor{orange!8} OperLoop (Loop 3) & $0.56\times$ & $1\times$ & 13.90 & 7.21 & 59.77 & 15.85 & 17.13 & 23.78 & 20.12  & 44.41 &32.18 & 25.58 & +1.69 \\
\midrule
HyperLoop (Loop 6) & $0.56\times$ & $1.67\times$ & 13.35 & 6.51 & 61.39 & 22.44 & 27.14 & 24.39 & \underline{20.73}  & 48.25 &35.54 & 29.75 & +5.86 \\
\rowcolor{orange!8} OperLoop (Loop 6) & $0.56\times$ & $1.67\times$ & 13.16 & \textbf{6.39} & 61.72 & \textbf{24.87} & \textbf{31.08} & \textbf{27.44} & \textbf{21.95}  & \textbf{50.15} &\textbf{35.89} & \textbf{31.90} & \textbf{+8.01} \\
\bottomrule
\end{tabular}%
}
\end{table}



\begin{wraptable}{r}{0.54\linewidth}
\vspace{-0.76cm}
\centering
\caption{Delta objective ablation at three loops with the causal step-size
schedule fixed.}
\label{tab:ablation-delta-target}
\renewcommand{\arraystretch}{1.15}
\resizebox{\linewidth}{!}{%
\begin{tabular}{@{}lcccccc@{}}
\toprule
Variant & \shortstack{Delta\\Obj.} & \shortstack{Step Size\\Schedule}
& \shortstack{Wiki.\\PPL $\downarrow$}
& \shortstack{Lamb.\\PPL $\downarrow$}
& \shortstack{LL\\Avg. $\uparrow$}
& \shortstack{EM\\Avg. $\uparrow$} \\
\midrule
\rowcolor{orange!8} OperLoop (Ours) & $\checkmark$ & Causal & \textbf{13.90} & \textbf{7.21} & \textbf{59.77} & \textbf{25.58} \\
w/o Delta Obj. & $\times$ & Causal & 13.92 & 7.25 & 59.36 & 24.84 \\
\bottomrule
\end{tabular}%
}
\end{wraptable}

\paragraph{Delta Objective Ablation.}
With the causal step-size schedule fixed, removing the delta objective worsens Wikitext PPL from $13.90$ to $13.92$ and Lambada PPL from $7.21$ to $7.25$, while lowering average commonsense accuracy from $59.77$ to $59.36$ and the EM average on generation tasks from $25.58$ to $24.84$ (Table~\ref{tab:ablation-delta-target}). Thus, the delta objective improves both language modeling and generative performance. This indicates that loop transition designs derived from an objective under the optimization view can improve performance.

\begin{wraptable}{r}{0.54\linewidth}
\vspace{-0.76cm}
\centering
\caption{Step-size schedule ablation at three loops with the delta objective fixed.}
\label{tab:ablation-lr}
\renewcommand{\arraystretch}{1.15}
\resizebox{\linewidth}{!}{%
\begin{tabular}{@{}lcccccc@{}}
\toprule
Variant & \shortstack{Delta\\Obj.} & \shortstack{Step Size\\Schedule}& \shortstack{Wiki.\\PPL $\downarrow$}
& \shortstack{Lamb.\\PPL $\downarrow$}
& \shortstack{LL\\Avg. $\uparrow$}
& \shortstack{EM\\Avg. $\uparrow$} 
\\
\midrule
\rowcolor{orange!8} OperLoop (Ours) & $\checkmark$ & Causal & \textbf{13.90} & 7.21& 59.77& \textbf{25.58} \\
LR w Non-causal & $\checkmark$ & Non-Causal &  13.93&  \textbf{7.15}&  \textbf{59.80}&  24.02\\
w/o LR & $\checkmark$ & $\times$&  13.95&  7.48&  59.38&  24.40\\
\bottomrule
\end{tabular}%
}
\end{wraptable}
\paragraph{Step-Size Schedule Ablation.}
With the delta objective fixed, we compare the causal schedule $\eta_l=\sigma(\cdot)\eta_{l-1}$, the non-causal schedule $\eta_l=\sigma(\cdot)$, and a fixed unit step $\eta=1$ (Table~\ref{tab:ablation-lr}). The causal and non-causal schedules use the state-dependent sigmoid factor in Section~\ref{sec:optimizer-guided-loop-design}. The causal schedule improves all reported metrics over unit steps, including EM by 1.18 percentage points. Without a monotonicity constraint, step-size fluctuations may induce oscillatory updates that hinder generation. Relative to the non-causal schedule, it improves the EM average by 1.56 percentage points. Since the step size is inspired by learning-rate scheduling, these performance gains support the optimization view.

\vspace{-0.4cm}
\section{Analysis and Discussion}
\vspace{-0.2cm}

Table~\ref{tab:loop_comparison_transposed} extends the analysis in §~\ref{sec:existing-transitions-coordinates} to additional loop models. We compare their \textit{projections} ($\boldsymbol{o}_l$), local \textit{objectives} ($\mathcal{L}$), and \textit{optimizer update rules} (weight decay $\lambda$, step size $\eta$, and input map $\boldsymbol{x}$) to identify directions for future transition design.
\begin{table}[t!]
\centering
\caption{Extended comparison of loop transitions under the optimization framework. Recurrent states are treated as fast weights corresponding to $\mathbf{W}$. Vector and expanded states are denoted by $\boldsymbol{y}_l\in\mathbb{R}^{d}$ and $\mathbf{Y}_l\in\mathbb{R}^{r\times d}$, respectively. Shared blocks predict the implicit target $\boldsymbol{t}_l=\operatorname{Blocks}_{\theta}(\boldsymbol{o}_l)$. {\color{darkred}Red} highlights input-map mismatches under the chosen optimizer update rule. Here, $\boldsymbol{e}$ is the Prelude output, $\boldsymbol{\epsilon}$ is a noise vector, and ${\ast}$ denotes token shifts along the context. For training-free Recirculation, we compare only its transition. The Vanilla category includes Universal Transformers~\citep{dehghani2019universal}, MoEUT~\citep{csordas2024moeut}, Recursive Transformer~\citep{bae2025recursive} and Ouro~\citep{zhu2025ouro}.}
\label{tab:loop_comparison_transposed}
\setlength{\tabcolsep}{1.1pt}
\normalsize
\resizebox{\linewidth}{!}{
\renewcommand{\arraystretch}{1.1}
\begin{tabular}{@{}llccccc@{}}
\toprule
 & & \multicolumn{5}{c}{$\mathbf{W}_{i+1}  =
(1-\eta\lambda)\,\mathbf{W}_{i}  -\eta\, (\boldsymbol{x}^{\top}f^{\prime}(\boldsymbol{o}_l))$}
\\
\cmidrule(lr){3-7}
& \textbf{Original Transition}
& $\boldsymbol{o}_l$
& $\lambda$
& $\eta$
& $\boldsymbol{x}$ 
& $\mathcal{L}$ 

\\

\midrule 

\makecell[l]{
\textbf{Vanilla}
}

&
$\boldsymbol{y}_{l+1}
=
\operatorname{Blocks}_{\theta}(\boldsymbol{o}_l)$
&
$\boldsymbol{y}_l$
&
$1$
&
$1$
&
$\mathbf{I}$
&
$-\langle\boldsymbol{t}_l,\boldsymbol{o}_l\rangle$

\\ 
\midrule

\makecell[l]{
\textbf{Huginn}
\\ 
{\scriptsize\citep{geiping2025scaling}}
}

&
$\boldsymbol{y}_{l+1}
=
\operatorname{Blocks}_{\theta}
(\boldsymbol{o}_l)$
&
$\boldsymbol{y}_l+\mathbf{e}$
&
$1$
&
$1$
&
$\mathbf{I}$
&
$-\langle\boldsymbol{t}_l,\boldsymbol{o}_l\rangle$
 
\\
\midrule

\makecell[l]{
\textbf{Parcae}\\
{\scriptsize\citep{prairie2026parcae}}
}
&
$\boldsymbol{y}_{l+1}
=
\operatorname{Blocks}_{\theta}
(\boldsymbol{o}_l)$
&
${\color{darkred}\bar{\mathbf{A}}}\boldsymbol{y}_l
+
\bar{\mathbf{B}}\mathbf{e}$
&
$1$
&
$1$
&
${\color{darkred}
\mathbf{I}}$
&
$-\langle\boldsymbol{t}_l,\boldsymbol{o}_l\rangle$
\\
\midrule

\makecell[l]{
\textbf{Full-Bandwidth}$^{\ast}$
\\
{\scriptsize\citep{wang2026full}}
}
&
$\boldsymbol{y}_{l+1}
=\operatorname{Blocks}_{\theta}(\boldsymbol{o}_l)
$
&
${\color{darkred}\sigma(\mathbf{W}^{\rm g} \boldsymbol{e}_l) \odot \mathbf{W}^{\rm u} }\boldsymbol{y}_l$
&
$1$
&
$1$
&
${\color{darkred}
\mathbf{I}
}$
&
$-\langle\boldsymbol{t}_l,\boldsymbol{o}_l\rangle$
\\
\midrule

\makecell[l]{
\textbf{RecurrentGPT}
\\
{\scriptsize\citep{hegazy2026recurrentgpt}}
}
&
$\boldsymbol{y}_{l+1}
=
\boldsymbol{g}_{l} \odot \boldsymbol{y}_{l}
+
( 1 - \boldsymbol{g}_{l}) \odot 
\operatorname{Blocks}_{\theta}(\boldsymbol{o}_l)$
&
${\color{darkred}\mathbf{W}} \ [\boldsymbol{y}_l + \boldsymbol{\epsilon}, \boldsymbol{e}]$
&
$1$
&
$1 - \boldsymbol{g}_{l}$
&
$
{\color{darkred}
\mathbf{I}
}$
&
$-\langle\boldsymbol{t}_l,\boldsymbol{o}_l\rangle$

\\

\midrule

\makecell[l]{
\textbf{Recirculation}$^{\ast}$
\\
{\scriptsize\citep{mozer2026recirculation}}
}
&
$\boldsymbol{y}_{l+1}
=
(1 - \alpha) \ \boldsymbol{y}_{l}
+
\alpha \|\boldsymbol{y}_l \|_2
 \frac{\operatorname{Blocks}_{\theta}(\boldsymbol{o}_l)}{\|\operatorname{Blocks}_{\theta}(\boldsymbol{o}_l)\|_2}$
&
$\boldsymbol{y}_l$
&
$\frac{1}{ \|\boldsymbol{y}_l \|_2}$
&
$ \alpha \|\boldsymbol{y}_l \|_2$
&
${
\mathbf{I}
}$
&
$-\langle\frac{\boldsymbol{t}_l}{\|\boldsymbol{t}_l\|_2}, \boldsymbol{o}_l\rangle$

\\

\midrule

\makecell[l]{
\textbf{HyperLoop}
\\
{\scriptsize\citep{zeitoun2026hyperloop}}
}
&
$\mathbf{Y}_{l+1}
=
\mathbf{H}^{\mathrm{res}}_l\mathbf{Y}_l
+
(\mathbf{H}^{\mathrm{post}}_l)^{\top}
\operatorname{Blocks}_{\theta}(\boldsymbol{o}_l)$
&
${\color{darkred}\mathbf{H}^{\mathrm{pre}}_l} \ \mathbf{Y}_l$
&
$\mathbf{I}-\mathbf{H}^{\mathrm{res}}_l$
&
$1$
&
${\color{darkred}
\mathbf{H}^{\mathrm{post}}_l}$
&
$-\langle\boldsymbol{t}_l,\boldsymbol{o}_l\rangle$

\\

\midrule

\makecell[l]{
\textbf{OperLoop}
\\
(Ours)
}
&
$
\begin{aligned}
\mathbf{Y}_{l+1}
&=
(\mathbf{I} - \eta_l \mathbf{\Lambda}_l) \mathbf{Y}_{l} \\
&+
\eta_l (\mathbf{H}_{l}^{in})^{\top}(
\operatorname{Blocks}_{\theta}(\boldsymbol{o}_l) - \boldsymbol{o}_l )
\end{aligned}
$
&
$\mathbf{H}_{l}^{\rm in} \mathbf{Y}_l$
&
$\mathbf{\Lambda}_l$
&
${\eta}_{l}$
&
$\mathbf{H}_{l}^{in}$
&
$\frac{1}{2} \| \boldsymbol{t}_l - \boldsymbol{o}_l \|^2_2$

\\
\bottomrule
\end{tabular}
}
\vspace{-0.25cm}
\end{table}

Input injection, gating, and expanded-state readouts define different \textit{projections}. Decay and step size control state retention and update scales in the \textit{optimizer update rule}. Under the chosen local gradient formulation, input-map mismatches occur in Parcae, Full-Bandwidth, RecurrentGPT, and HyperLoop. The framework identifies which transitions need input-map alignment to match the chosen optimizer update rule.

Most compared transitions map to a negative inner product \textit{objective}, with Recirculation using a normalized target. Their write term is driven by the predicted target. OperLoop's delta objective instead uses the difference between the target and projected state, adding an explicit correction for the current projection. The ablation in §~\ref{sec:experiments} supports this objective choice, illustrating how the framework guides the design of a different update signal.

This comparison suggests a roadmap: specify the \textit{projection} and local \textit{objective}, select an \textit{optimizer update rule}, and derive the loop transition. Future work can vary the structure, computation, and initialization of fast weights and explore linear or nonlinear \textit{projections}. \textit{Objectives} shape update signals, and their values may guide adaptive looping for early stopping or extrapolation. Ideas from advanced \textit{optimizers} like Muon~\citep{jordan2024muon} or Hyperball~\citep{wen2026fantastic} may inspire new loop architectures to improve loop stability and reasoning performance.

\vspace{-0.2cm}
\section{Conclusion}
\vspace{-0.2cm}

In this work, we present an optimization view of Looped Transformers that treats recurrent states as fast weights updated under local prediction objectives. This framework organizes loop transitions into a projection, an objective, and an optimizer rule, with shared blocks predicting implicit targets. It identifies input-map mismatches in existing transitions and guides the design of OperLoop. The alignment studies, downstream evaluations, and component ablations support its design guidance. The extended comparison further identifies directions for improving loop transitions in the future.


\newpage

\newpage

\subsection*{AI use statement}


In this work, we used generative AI tools for language polishing, grammar correction, and figure creation to improve the clarity and readability of the manuscript. We also used these tools to help identify related literature. We did not use generative AI tools to generate original research ideas. We have reviewed all AI-assisted work, including the edited text, figures, and suggested references, to ensure accuracy and consistency with our research. We take responsibility for the final content of this work, including text, claims, or artifacts produced with the aid of generative AI.

\subsubsection*{Acknowledgments}
This work is supported in part by the Guangdong Basic and Applied Basic Research Foundation (NO.2025A1515011758) and  Guangdong Provincial Key Lab of Integrated Communication, Sensing and Computation for Ubiquitous Internet of Things (No. 2023B1212010007).

\bibliography{iclr2027_conference}
\bibliographystyle{iclr2027_conference}

\newpage
\appendix

\section{Extended Related Work}\label{apx.extended.related.work}
\paragraph{Origins and development of Looped Transformers.}
Looped Transformers build on parameter sharing across depth and repeated application of a shared transition. Adaptive Computation Time makes the number of recurrent updates input-dependent~\citep{graves2016adaptive}. Universal Transformers repeatedly apply a shared Transformer block with adaptive per-position computation~\citep{dehghani2019universal}. ALBERT improves parameter efficiency through cross-layer sharing~\citep{lan2020albert}. Later work studies length generalization with adaptive loop counts~\citep{fan2025looped}. RecurrentGPT adds recurrent modulation to a shared Transformer core~\citep{hegazy2026recurrentgpt}. Training-free methods introduce inference-time looping into frozen pretrained models without additional training~\citep{chen2026training}. Looping thus increases effective depth without adding new parameters at each step.
\vspace{-0.1cm}

\paragraph{Looping, Test-Time Compute, and Reasoning.}
Looping can scale test-time computation through additional recurrent passes, as demonstrated by Huginn~\citep{geiping2025scaling}. Broader work on test-time scaling studies verifier-guided search and iterative response revision~\citep{snell2024scaling}. Ouro studies latent reasoning through repeated hidden-state computation~\citep{zhu2025ouro}. Attractor Models refine embeddings by solving for a fixed point~\citep{fein2026solve}. Saunshi et al.\ analyze the reasoning capacity of looped Transformers and establish theoretical connections to chain-of-thought (CoT)~\citep{saunshi2025reasoning}. Rodkin et al.\ study recurrence, memory, and test-time computation in controlled cellular-automata tasks~\citep{rodkin2026beyond}. CoT prompting elicits explicit intermediate reasoning steps~\citep{wei2022chain,kojima2022large}, whereas looping refines hidden states. The two approaches therefore provide complementary routes to test-time reasoning.
\vspace{-0.1cm}

\paragraph{Recurrent Updates Along Depth and Context.}
Looped Transformers apply recurrence along depth, updating the state at each visit to a shared block for a fixed input. Residual connections propagate representations across layers in ResNets and Transformers~\citep{he2016deep,vaswani2017attention}. Looping additionally shares block parameters across recurrent visits. Context-axis recurrence instead updates the state as tokens are processed. Examples include recurrent networks~\citep{hochreiter1997long}, Fast Weight and Linear-Attention models~\citep{ba2016using,schlag2021transformers}, and Structured State Space models~\citep{gu2021efficiently,dao2024transformers}. Delta-based linear attention models update recurrent memory along context using delta-rule corrections~\citep{yang2025gated,huang2026mdn,kimiteam2025kimilinearexpressiveefficient}. These updates offer useful analogies for loop transitions, but follow a different recurrence axis.
\vspace{-0.1cm}
\paragraph{Optimization Views Beyond Looped Transformers.}
Viewing recurrent updates as optimization updates connects our framework to work on neural computation beyond looped Transformers. LISTA unfolds iterative shrinkage and thresholding for sparse coding into a trainable network~\citep{gregor2010learning}. \citet{li2018optimization} interprets feedforward computation with shared linear transformations as gradient descent. They use accelerated optimization algorithms to motivate network architectures. Learned optimizers use recurrent networks to learn parameter-update rules~\citep{andrychowicz2016learning}. Studies of linear regression show that Transformers can implement gradient-based in-context learning in their forward pass~\citep{akyurek2023what,vonoswald2023transformers}. Test-Time Training (TTT) layers represent the recurrent hidden state as model parameters. As tokens are processed, gradient descent updates these parameters using a self-supervised objective~\citep{sun2024learning}. \citet{yang2024looped} study looped Transformers as iterative learning algorithms for in-context learning, while \citet{gatmiry2024can} show that, in linear in-context regression, a trained looped Transformer can implement multi-step preconditioned gradient descent. Our framework takes the reverse direction: we use an optimizer update as a construction principle to derive depth-axis loop transitions from a projection and a local objective. This local objective defines the forward transition and differs from the language-modeling loss used for end-to-end training.

\section{Pseudocode for OperLoop}\label{apx.pseudocode}

Algorithm~\ref{alg:hyperloop-delta-causal-lr} gives a common implementation of HyperLoop and OperLoop. Blue and orange denote the HyperLoop and OperLoop transition rules, respectively. Both initialize the recurrent state by duplicating the Prelude output into parallel streams. At each loop, OperLoop computes a state-dependent input map, decay operator, and causal step size. These determine the read, write, and residual maps. The residual map incorporates the delta correction in Eq.~\ref{eq.after.Y_l+1}. After $R$ recurrent steps, the streams are averaged and passed through the Coda and language-model head.

\paragraph{State layout and module correspondence.}
All HyperLoop and OperLoop experiments use $r=4$ parallel streams. The state $\mathbf{Y}_l\in\mathbb{R}^{r\times d}$ represents one token, with batch and token indices omitted. Flattening and RMSNorm produce the controller input $\mathbf{Z}_l\in\mathbb{R}^{rd}$. The map $\mathbf{H}^{\mathrm{in}}_l\in\mathbb{R}^{1\times r}$ reads the streams into the block input. The decay operator $\mathbf{\Lambda}_l\in\mathbb{R}^{r\times r}$ acts on the stream dimension. Each $\operatorname{RecurrentBlock}_{\theta}$ call represents the shared stack $\operatorname{Blocks}_{\theta}$ of $K$ Transformer blocks.

\begin{algorithm}[t!]
\small
\caption{Common implementation of \algblue{HyperLoop}~\citep{zeitoun2026hyperloop} and \algred{OperLoop} (\textbf{ours}).}
\label{alg:hyperloop-delta-causal-lr}
\begin{algorithmic}[1]
\Require
Input tokens $\mathbf{x}$;
recurrent-step index $l$;
number of streams $r$;
hidden dimension $d$;
number of loops $R$; Trainable bias $\boldsymbol{e}_l \in \{\boldsymbol{e}_l\}_{l=0,\cdots,R-1}$; Trainable weights $\{\mathbf{W}^{\tau}_l, \boldsymbol{b}^{\tau}_l,\boldsymbol{a}^{\tau}_l\}_{l=0,\cdots,R-1}$, \algblue{$\tau \in \{\text{res, pre, post}\}$} or \algred{$\tau \in \{\text{wd, in, lr}\}$}.

\Ensure Language-model head output $\operatorname{LMHead}_{\theta}(\mathbf{h}_{\mathrm{out}})$.


\State $\mathbf{h} \gets \operatorname{Embedding}_{\theta}(\mathbf{x})$
\State $\mathbf{h} \gets \operatorname{Prelude}_{\theta}(\mathbf{h})$

\State $\mathbf{Y}_0
  \gets \operatorname{Duplicate}_{r}(\mathbf{h})$ \Comment{$\in \mathbb{R}^{d} \to \mathbb{R}^{r \times d}$}

\State \algred{$\eta_{-1} \gets \mathbf{1}$}

\For{$l \gets 0$ \textbf{to} $R-1$}

  \State $\mathbf{Z}_l
      \gets \operatorname{RMSNorm}(\operatorname{flatten}(\mathbf{Y}_l))$ \Comment{$\in \mathbb{R}^{r \times d} \to \mathbb{R}^{rd}$}

  \State \algred{$
      \eta_l
      \gets
       \eta_{l-1} \cdot \sigma(
        {a}^{\mathrm{lr}}
        \cdot
        \left(
        \mathbf{W}^{\mathrm{lr}} \mathbf{Z}_l
        \right)
        +
        {b}^{\mathrm{lr}}
        ) $}  \Comment{$ \in \mathbb{R}$}

  \State \algred{$
      \mathbf{H}^{\mathrm{in}}_l
      \gets
        \sigma(
        \boldsymbol{a}^{\mathrm{in}}
        \cdot
        \left(
        \mathbf{W}^{\mathrm{in}} \mathbf{Z}_l
        \right)
        +
        \boldsymbol{b}^{\mathrm{in}}
        )$}  \Comment{$ \in \mathbb{R}^{1 \times r}$}

  \State \algred{$
      \mathbf{\Lambda}_l
      \gets
      \operatorname{Diag}
        \left(
        \sigma\left(
        \boldsymbol{a}^{\mathrm{wd}}
        \cdot
        \left(
        \mathbf{W}^{\mathrm{wd}} \mathbf{Z}_l
        \right)
        +
        \boldsymbol{b}^{\mathrm{wd}}
        \right)
        \right)$} \Comment{$\in \mathbb{R}^{r \times r} $}


  \State $
      \mathbf{H}^{\mathrm{pre}}_l
      \gets
      \algbluemath{
        \sigma\left(
        \boldsymbol{a}^{\mathrm{pre}}_l
        \cdot
        \left(
        \mathbf{W}^{\mathrm{pre}}_l \mathbf{Z}_l
        \right)
        +
        \boldsymbol{b}^{\mathrm{pre}}_l
        \right) }  \text{ or }
      \algredmath{\mathbf{H}^{\mathrm{in}}_l}$ \Comment{$\in \mathbb{R}^{1 \times r} $}

  \State  $\mathbf{H}^{\mathrm{post}}_l
      \gets
        \algbluemath{
        2 \cdot \sigma\left(
        \boldsymbol{a}^{\mathrm{post}}_l
        \cdot
        \left(
        \mathbf{W}^{\mathrm{post}}_l \mathbf{Z}_l
        \right)
        +
        \boldsymbol{b}^{\mathrm{post}}_l
        \right) }  \text{ or }
      \algredmath{\eta_l \mathbf{H}^{\mathrm{in}}_l}$\Comment{$\in \mathbb{R}^{1 \times r} $}

  \State $
      \mathbf{H}^{\mathrm{res}}_l
      \gets
      \algbluemath{\operatorname{Diag}
        \left(
        \sigma\left(
        \boldsymbol{a}^{\mathrm{res}}_l
        \cdot
        \left(
        \mathbf{W}^{\mathrm{res}}_l \mathbf{Z}_l
        \right)
        +
        \boldsymbol{b}^{\mathrm{res}}_l
        \right)
        \right)}  \text{ or }
      \algredmath{\mathbf{I}
      -
      \eta_l \cdot
      \left(
          \mathbf{\Lambda}_l
          +
          \left(\mathbf{H}^{\mathrm{in}}_l\right)^{\top}
          \mathbf{H}^{\mathrm{in}}_l
      \right)}$ \Comment{$\in \mathbb{R}^{r \times r} $}


  \State $\mathbf{Y}_{l+1}
      \gets
      \mathbf{H}^{\mathrm{res}}_l
      \mathbf{Y}_{l} +
      (\mathbf{H}^{\mathrm{post}}_l)^{\top} \cdot
      (\operatorname{RecurrentBlock}_{\theta}(\mathbf{H}^{\mathrm{pre}}_l
      \mathbf{Y}_l) + \boldsymbol{e}_l)$ \Comment{$\in \mathbb{R}^{r \times d} $}

\EndFor

\State $\mathbf{h}_{\mathrm{loop}}
  \gets
  {\mathbf{Y}_R}.\operatorname{Mean}(-2)$  \Comment{$\in \mathbb{R}^{r \times d} \to \mathbb{R}^{d}$}

\State $\mathbf{h}_{\mathrm{out}}
  \gets
  \operatorname{Coda}_{\theta}(\mathbf{h}_{\mathrm{loop}})$

\State \Return $\operatorname{LMHead}_{\theta}(\mathbf{h}_{\mathrm{out}})$

\end{algorithmic}
\end{algorithm}

\vspace{-0.2cm}
\paragraph{Implementation details.}
HyperLoop and OperLoop use loop-specific transition parameters. Their loop indices are omitted in the OperLoop expressions. Both retain HyperLoop's trainable output bias $\boldsymbol{e}_l$. In the implementation, it shifts the predicted target to $\boldsymbol{t}_l=\operatorname{Blocks}_{\theta}(\boldsymbol{o}_l)+\boldsymbol{e}_l$. The optimizer-derived formulas omit this term for brevity. This matches HyperLoop's bias setting when assessing performance gains from the optimizer-guided loop transition. Parcae uses shared transition parameters and initializes its recurrent state with Gaussian noise. All recurrent steps are fully backpropagated. Further details of the local optimization interpretation appear in Appendix~\ref{app:transition-details}.
\vspace{-0.2cm}
\paragraph{Causal step-size schedule.}
Let $\gamma_l$ denote the sigmoid factor in the step-size parameterization in §~\ref{sec:optimizer-guided-loop-design}. With $\eta_{-1}=1$, the causal recurrence gives $\eta_l=\gamma_l\eta_{l-1}=\prod_{j=0}^{l}\gamma_j$. Since $0<\gamma_l<1$ for finite sigmoid inputs, the step size decreases monotonically across recurrent depth. Here, ``causal'' refers to dependence on the previous loop's step size. The non-causal ablation uses $\eta_l=\gamma_l$, and the unit-step ablation fixes $\eta=1$. Both retain the delta objective in Table~\ref{tab:ablation-lr}.

\section{Model and Training Configuration}
\label{app:configuration}
\vspace{-0.2cm}
Table~\ref{tab:architecture-config} summarizes the architecture of the 18- and 32-layer baselines. Both use the same hidden width, attention configuration, and MoE design. Sliding-window and full attention follow a repeating 3:1 pattern. The first two layers are dense, and the remaining layers use MoE. Parameter estimates refer to trainable non-vocabulary parameters.

\begin{table}[t!]
\centering
\caption{Architecture configuration used in the experiments. The sliding--sliding--sliding--full pattern repeats every four layers.}
\label{tab:architecture-config}
\scriptsize
\setlength{\tabcolsep}{4pt}
\renewcommand{\arraystretch}{1.1}
\resizebox{\textwidth}{!}{%
\begin{tabular}{@{}lll@{}}
\toprule
\textbf{Component} & \textbf{Setting} & \textbf{Value} \\
\midrule
Model family & Architecture & Decoder-only Transformer \\
Numerics & Parameter / router / LM-head precision & BF16 / FP32 / FP32 \\
Representation & Hidden size / vocabulary size & $1024$ / $128{,}896$ \\
Attention & Heads / head dimension / GQA groups & $16$ / $128$ / $4$ \\
Context & Maximum sequence length & $4096$ \\
Attention pattern & Repeated every four layers & Sliding--Sliding--Sliding--Full \\
Sliding attention & Window / QK RoPE dimension / RoPE theta & $512$ / $128$ / $1{,}000$ \\
Full attention & QK RoPE dimension / RoPE theta & $64$ / $10{,}000$ \\
Feed-forward & Dense intermediate size / activation & $4096$ / SwiGLU \\
MoE & Type / experts per MoE layer / router top-$k$ & Marco-style / $256$ / $8$ \\
MoE experts & Routed / shared intermediate size & $384$ / $384$ \\
Routing & Router / load balancing & Sigmoid / auxiliary-loss-free \\
Routing & Routed scaling / router bias update rate / coefficient & $2.5$ / $5\times10^{-3}$ / $10^{-3}$ \\
Initialization & Standard deviation / LayerNorm epsilon & $0.006$ / $10^{-5}$ \\
Initialization & RMSNorm gamma initialization & Zero \\
\midrule
18-layer model & Layers / dense layers / MoE layers & $18$ / $0$--$1$ / $2$--$17$ \\
18-layer model & Full-attention layers / estimated parameters & $3,7,11,15$ / $4.974$B \\
32-layer model & Layers / dense layers / MoE layers & $32$ / $0$--$1$ / $2$--$31$ \\
32-layer model & Full-attention layers / estimated parameters & $3,7,11,15,19,23,27,31$ / $9.296$B \\
\bottomrule
\end{tabular}%
}
\end{table}

\vspace{-0.2cm}
\paragraph{Loop layouts and compute scaling.}
Table~\ref{tab:loop-layout} details the middle-loop configurations in Tables~\ref{tab:main-results} and~\ref{tab:ablation-loop-count}. Each layout contains $i$ Prelude blocks, $K$ shared blocks, and $j$ Coda blocks. Repeating the shared blocks for $R$ steps gives an effective depth of $i+RK+j$. The number of distinct Transformer blocks is $i+K+j$. Parameter and FLOP totals also depend on the block contents and transition overhead. Loop3 and Loop6 are trained separately and evaluated at their respective training depths. Their comparison therefore changes both training and inference computation.

\begin{table}[t!]
\centering
\caption{The number of distinct blocks is $i+K+j$, and the effective depth is $i+RK+j$.}
\label{tab:loop-layout}
\small
\renewcommand{\arraystretch}{1.1}
\resizebox{\linewidth}{!}{%
\begin{tabular}{lcccccc}
\toprule
Setting & Prelude $i$ & Shared $K$ & Loops $R$ & Coda $j$ & Effective Depth & Distinct Blocks \\
\midrule
18-layer comparison (Loop3) & 4 & 4 & 3 & 2 & 18 & 10 \\
32-layer comparison (Loop3) & 4 & 8 & 3 & 4 & 32 & 16 \\
Compute scaling (Loop6) & 4 & 4 & 6 & 2 & 30 & 10 \\
\bottomrule
\end{tabular}%
}
\end{table}
\vspace{-0.2cm}
\paragraph{Training settings.}
Table~\ref{tab:training-config} reports the training settings. Both scales use the same batch configuration, weight decay, gradient clipping, and cosine learning-rate schedule. Muon updates eligible QKV and FFN linear weights, while AdamW updates the remaining parameters. The peak learning rate, warmup duration, and number of updates vary by scale.
\vspace{-0.2cm}
\paragraph{Training budgets.}
Nominal token counts are calculated from the sequence length, global batch size, and number of updates. A global batch of $1024$ sequences of length $4096$ gives $4{,}194{,}304$ nominal token positions per update. The $67{,}000$ and $120{,}000$ updates correspond to $281.018$B and $503.316$B nominal tokens, respectively. The main text uses the approximate budget labels $300$B and $500$B. Comparisons at matched compute use training FLOPs, including transition overhead. Parameter ratios are normalized and rounded; they do not imply exact parameter equality.

\begin{table}[t!]
\centering
\caption{Training configuration. Nominal token counts are computed from the sequence length, global batch size, and number of updates.}
\label{tab:training-config}
\scriptsize
\setlength{\tabcolsep}{4pt}
\renewcommand{\arraystretch}{1.1}
\resizebox{\textwidth}{!}{%
\begin{tabular}{@{}lll@{}}
\toprule
\textbf{Category} & \textbf{Setting} & \textbf{Value} \\
\midrule
Objective & Training objective / loss & Causal LM / vocabulary-parallel cross-entropy \\
Batching & Global batch / micro batch & $1024$ / $8$ sequences \\
Batching & Nominal tokens per update & $4{,}194{,}304$ \\
Precision & Training parameter precision & BF16 \\
Optimizer & Muon momentum / Nesterov & $0.95$ / enabled \\
Optimizer & Newton--Schulz variant / iterations & Polar Express / $6$ \\
Optimizer & AdamW $\beta_1$ / $\beta_2$ / $\epsilon$ & $0.9$ / $0.95$ / $10^{-8}$ \\
Regularization & Weight decay / global gradient clipping & $0.1$ / $1.0$ \\
Schedule & Learning-rate schedule / minimum rate & Cosine / $10^{-5}$ \\
Schedule & Schedule unit & Iteration \\
\midrule
18-layer model & Peak rate / warmup / iterations & $7.3\times10^{-4}$ / $1{,}000$ / $67{,}000$ \\
18-layer model & Nominal training tokens / Muon-matched AdamW RMS & $281.018$B / $0.18$ \\
32-layer model & Peak rate / warmup / iterations & $5.6\times10^{-4}$ / $2{,}000$ / $120{,}000$ \\
32-layer model & Nominal training tokens / Muon-matched AdamW RMS & $503.316$B / $0.18$ \\
\bottomrule
\end{tabular}%
}
\end{table}

\section{Evaluation Details}
\label{app:evaluation}

We evaluate perplexity and downstream performance with \texttt{lm-eval-harness} v0.4.13~\citep{gao2023framework}. Each benchmark uses its task-specific default prompts, few-shot settings, answer filters, normalization rules, and generation parameters. Table~\ref{tab:main-results} summarizes the few-shot and chain-of-thought (CoT) settings. All compared models use the same evaluation configuration.

We report perplexity on Wikitext~\citep{merity2016pointer} and Lambada~\citep{radford2019language}. Commonsense scores cover Lambada)~\citep{radford2019language}, ARCc and ARCe~\citep{clark2018think}, HellaSwag~\citep{zellers2019hellaswag}, WinoGrande~\citep{sakaguchi2019winogrande}, and PIQA~\citep{bisk2019piqa}.

For GSM8K~\citep{cobbe2021training}, we report 5-shot evaluation and 8-shot evaluation with CoT. HumanEval~\citep{chen2021evaluating} and HumanEval+~\citep{liu2023evalplus} use 0-shot evaluation and report pass@1. MMLU~\citep{hendrycks2021measuring} uses 5-shot evaluation, with generated answers scored by exact match (EM). BBH~\citep{suzgun2022challenging} uses 3-shot evaluation with CoT.

Downstream scores are reported as percentages. LL Avg. is the unweighted mean of the six reported commonsense scores. EM Avg. is the unweighted mean of the six reported math, code, and reasoning scores, including code pass@1. The two GSM8K settings and the two HumanEval variants each contribute separate entries. EM \& LL Avg. is the mean of the two reported group averages, with equal weight assigned to each group. Accuracy differences and $\Delta$ in Table~\ref{tab:ablation-loop-count} are measured in percentage points. All compared runs use the same training and evaluation seed. Every result is evaluated from the final checkpoint of its run.

\section{Additional Details of the Optimization View}
\label{app:transition-details}

The following identities detail the local-gradient calculation in Proposition~\ref{proposition}. At recurrent step $l$, consider $\boldsymbol{o}_l=\mathbf{H}^{\mathrm{in}}_l\mathbf{Y}_l+\boldsymbol{b}_l$. For the local derivative with respect to $\mathbf{Y}_l$, hold $\mathbf{H}^{\mathrm{in}}_l$, $\boldsymbol{b}_l$, and the predicted target $\boldsymbol{t}_l$ fixed. The linear and delta objectives then give:
\begin{align}
\mathcal{L}^{\mathrm{lin}}_l&=-\langle\boldsymbol{o}_l,\boldsymbol{t}_l\rangle,
&\nabla_{\mathbf{Y}_l}\mathcal{L}^{\mathrm{lin}}_l&=-(\mathbf{H}^{\mathrm{in}}_l)^{\top}\boldsymbol{t}_l,\label{eq:app-linear-gradient}\\
\mathcal{L}^{\mathrm{delta}}_l&=\tfrac{1}{2}\|\boldsymbol{o}_l-\boldsymbol{t}_l\|_2^2,
&\nabla_{\mathbf{Y}_l}\mathcal{L}^{\mathrm{delta}}_l&=(\mathbf{H}^{\mathrm{in}}_l)^{\top}(\boldsymbol{o}_l-\boldsymbol{t}_l).\label{eq:app-delta-gradient}
\end{align}
Holding these quantities fixed is only a convention for deriving the closed-form forward transition. The target and state-dependent maps are recomputed across recurrent steps. End-to-end training does not detach the target, input map, or bias. Gradients propagate through their computation.

For the bias-free projection in Eq.~\ref{eq.pre.Y_l+1}, substitute $\boldsymbol{o}_l=\mathbf{H}^{\mathrm{in}}_l\mathbf{Y}_l$ into the delta update. This gives
\begin{align}
\mathbf{Y}_{l+1}
&=(\mathbf{I}-\eta_l\mathbf{\Lambda}_l)\mathbf{Y}_l
 +\eta_l(\mathbf{H}^{\mathrm{in}}_l)^{\top}(\boldsymbol{t}_l-\mathbf{H}^{\mathrm{in}}_l\mathbf{Y}_l)\notag\\
&=\bigl(\mathbf{I}-\eta_l\bigl(\mathbf{\Lambda}_l+(\mathbf{H}^{\mathrm{in}}_l)^{\top}\mathbf{H}^{\mathrm{in}}_l\bigr)\bigr) \mathbf{Y}_l
 +\eta_l(\mathbf{H}^{\mathrm{in}}_l)^{\top}\boldsymbol{t}_l.\label{eq:app-expanded-delta}
\end{align}
This recovers the residual-map form in Eq.~\ref{eq.after.Y_l+1}, with $\mathbf{H}^{\mathrm{pre}}_l=\mathbf{H}^{\mathrm{in}}_l$ and $\mathbf{H}^{\mathrm{post}}_l=\eta_l\mathbf{H}^{\mathrm{in}}_l$. The stored write map therefore includes the step-size factor. The HyperLoop alignment study equates the stored input and write maps under a unit step size. OperLoop uses the adaptive step size $\eta_l$.

Here, a closed-form transition specifies one update of the looped state. It does not denote a closed-form minimizer of the local objective. The local objective differs from the causal language-modeling loss used to train model parameters. The loop step size $\eta_l$ differs from the training learning rate. The decay operator $\mathbf{\Lambda}_l$ also differs from the training weight-decay coefficient in Table~\ref{tab:training-config}.





\end{document}